\documentclass{article}

\usepackage[main, final]{neurips_2026}
\setcitestyle{numbers,square,comma}

\usepackage[utf8]{inputenc}
\usepackage[T1]{fontenc}
\usepackage{hyperref}
\hypersetup{hidelinks}
\usepackage{url}
\usepackage{booktabs}
\usepackage{multirow}
\usepackage{amsfonts}
\usepackage{nicefrac}
\usepackage{microtype}
\usepackage{xcolor}
\usepackage{amsmath}
\usepackage{amssymb}
\usepackage{mathtools}
\usepackage{amsthm}
\usepackage{xcolor-material}
\usepackage{subcaption}
\usepackage{bbm}
\usepackage{tikz}
\usetikzlibrary{arrows.meta}

\definecolor{commentblue}{HTML}{598be7}
\definecolor{TrainingColor}{HTML}{9B4F96}

\usepackage{algorithm}
\usepackage[endLComment={},commentColor=commentblue]{algpseudocodex}

\usepackage[percent]{overpic}
\usepackage{enumitem}

\setlist[itemize]{leftmargin=*, labelindent=0pt}
\setlist[enumerate]{leftmargin=*, labelindent=0pt}

\usepackage{stmaryrd}
\usepackage[pgf]{trimclip}

\usepackage{centernot}

\makeatletter
\DeclareRobustCommand{\shortto}{\mathrel{\mathpalette\short@to\relax}}

\newcommand{\short@to}[2]{\mkern2mu
  \clipbox{{.4\width} 0 0 0}{$\m@th#1\vphantom{+}{\shortrightarrow}$}}
\makeatother

\let\emailttdefault\ttdefault
\DeclareTextFontCommand{\emailtexttt}{\fontfamily{\emailttdefault}\selectfont}
\usepackage[scaled=0.88]{beramono}

\usepackage{cleveref}

\theoremstyle{plain}
\newtheorem{theorem}{Theorem}[section]

\newtheorem{lemma}[theorem]{Lemma}

\theoremstyle{definition}

\theoremstyle{remark}
\newtheorem{remark}[theorem]{Remark}

\def\eqref#1{equation~\ref{#1}}
\newcommand{\E}{\mathbb{E}}
\newcommand{\KL}{D_{\mathrm{KL}}}
\DeclareMathOperator*{\argmax}{arg\,max}

\definecolor{plgray}{HTML}{888888}
\DeclareMathOperator{\supp}{supp}

\newcommand{\pl}[1]{{\color{plgray} #1}}
\newcommand{\algo}{\texttt{LASER}}
\newcommand{\ourname}[0]{\textcolor{MaterialBlue900}{{\algo{}}}}
\newcommand{\baselinename}[1]{\textcolor{MaterialPurple900}{\texttt{#1}}}
\newcommand{\datasetname}[1]{{\textsc{#1}}}
\newcommand{\envname}[1]{{\fontfamily{lmss}\selectfont#1}}
\renewcommand{\eqref}[1]{\textup{(\ref{#1})}}

\definecolor{BCBlue}{HTML}{3D85BF}
\definecolor{NGGrey}{HTML}{AEAEAE}
\definecolor{FTPurple}{HTML}{7E35C5}
\DeclareRobustCommand{\bcarrow}{\tikz[baseline=-0.7ex] \draw[-{Stealth[inset=0pt, length=5pt, width=4pt]}, color=BCBlue, line width=1.2pt] (0,0) -- (0.5,0);}
\DeclareRobustCommand{\ngarrow}{\tikz[baseline=-0.7ex] \draw[-{Stealth[inset=0pt, length=5pt, width=4pt]}, color=NGGrey, line width=1.2pt] (0,0) -- (0.5,0);}
\DeclareRobustCommand{\amarrow}{\tikz[baseline=-0.7ex] \draw[-{Stealth[inset=0pt, length=5pt, width=4pt]}, color=FTPurple, line width=1.2pt] (0,0) -- (0.5,0);}

\DeclarePairedDelimiter{\norm}{\lVert}{\rVert}

\newcommand*\diff{\mathop{}\!\mathrm{d}}

\newcommand{\Gaussian}{\mathcal{N}}
\newcommand{\Uniform}{\mathcal{U}}

\newcommand{\ActionSpace}{\mathcal{A}}
\newcommand{\LatentSpace}{\mathcal{L}}
\newcommand{\NoiseSpace}{\mathcal{B}}

\newcommand{\NToL}{{\textsf{B} \shortto \textsf{L}}}
\newcommand{\LToA}{{\textsf{L} \shortto \textsf{A}}}

\newcommand{\BaseVelocity}{v_{\NToL}}

\newcommand{\BehaviorPol}{\pi_\beta}

\newcommand{\LatentBCDist}{q_\textsf{L}}
\newcommand{\LatentPolicy}{\pi_\textsf{L}}
\newcommand{\LatentUniform}{q_\textsf{L}^\mathrm{unif}}
\newcommand{\LatentUniformBold}{\mathbf{q}_\textsf{\textbf{L}}^\mathrm{\mathbf{unif}}}
\newcommand{\FullPolicy}{\pi_\textsf{A}}
\newcommand{\BaseDist}{q_\textsf{B}}

\newcommand{\thetaNToL}{{\theta_{\textsf{B} \shortto \textsf{L}}}}
\newcommand{\thetaLToA}{{\theta_{\textsf{L} \shortto \textsf{A}}}}
\newcommand{\thetaLToAPrime}{{\theta'_{\textsf{L} \shortto \textsf{A}}}}

\newcommand{\muNToLBase}{{\mu_{\NToL}^{\mathrm{base}}}}
\newcommand{\vNToLBase}{{v_{\NToL}^{\mathrm{base}}}}
\newcommand{\psiNToLBase}{{\psi_{\NToL}^{\mathrm{base}}}}

\newcommand{\muNToL}{{\mu_{\thetaNToL}}}
\newcommand{\muNToLCondS}{{\mu_{\thetaNToL}^s}}

\newcommand{\muLToA}{{\mu_{\thetaLToA}}}
\newcommand{\muLToACondS}{{\mu_{\thetaLToA}^s}}

\title{LASER: Latent Space Adjoint Matching for Support-Constrained Entropy-Regularized Offline RL}

\author{Songyuan Zhang\textsuperscript{1} \quad
  Oswin So\textsuperscript{1} \quad
  Eric Yang Yu\textsuperscript{1} \quad
  Matthew Cleaveland\textsuperscript{2} \quad
  \\
  \bfseries Peter Crowley-Dolen\textsuperscript{2} \quad
  Chuchu Fan\textsuperscript{1} \\[0.5ex]
  \normalfont\textsuperscript{1}MIT \qquad
  \textsuperscript{2}MIT Lincoln Laboratory \\[0.5ex]
  \emailtexttt{\{szhang21,oswinso,eyyu,chuchu\}@mit.edu} \\
  \emailtexttt{\{matthew.cleaveland,Peter.Crowley-Dolen\}@ll.mit.edu}
}

\begin{document}

\maketitle

\begingroup
\renewcommand{\thefootnote}{}
\footnotetext[0]{DISTRIBUTION STATEMENT A. Approved for public release. Distribution is unlimited.}
\endgroup

\begin{abstract}
    While offline reinforcement learning (RL) enables policy optimization from static datasets without costly online interaction, it remains bottlenecked by the risk of executing out-of-distribution (OOD) actions.
    Recent approaches mitigate this by learning a behavior-cloning policy through flow matching and then performing RL within its constrained latent space.
    However, na\"ively optimizing the latent policy can easily cause the policy to collapse into a brittle mode or exploit sharp artifacts of the learned critic.
    In this work, we find that entropy regularization is essential in latent-space RL for addressing these challenges.
    We introduce \algo{}, a novel offline RL algorithm that applies latent-space adjoint matching to achieve entropy-regularized latent-space RL with expressive flow policies while avoiding backpropagation through time.
    Through comprehensive experiments on 40 challenging OGBench tasks with varying dataset qualities, we show that \algo{} achieves state-of-the-art performance. Notably, \algo{} uses
    fixed method-specific hyperparameters across all tasks and outperforms the evaluated baselines, including those with task- and dataset-specific tuning, which highlights the robust applicability of \algo{}.
    Project website: \url{https://mit-realm.github.io/laser/}.
\end{abstract}

\section{Introduction}
Offline reinforcement learning (RL) aims to learn policies that improve upon the behavior policy using only a static dataset generated by that behavior policy, without further environment interaction \citep{levine2020offline}.
This creates a central tension: improvement requires deviating from the behavior policy while choosing actions where the dataset provides reliable supervision.
For a fixed state, actions outside the dataset support may receive spuriously high predicted values, which can be exploited by policy optimization to yield out-of-distribution (OOD) actions that perform poorly when deployed \citep{levine2020offline}.

A major line of work regularizes a statistical distance between the learned and behavior policies to mitigate OOD actions~\citep{wang2018exponentially, peng2019advantage, mao2023str, kumar2019stabilizing, wu2019behavior}.
However, depending on the distance used, these methods can be either too restrictive and produce an overly conservative policy (e.g., KL divergence), or not restrictive enough and still produce OOD actions (e.g., Wasserstein distance, MMD distance) \citep{zhang2026reform}.

This has motivated a more direct alternative: enforce \textbf{support constraints} on the policy extraction process to guarantee that the actions produced lie within the support of the action distribution of the behavior policy \citep{wagenmaker2025steering,zhang2026reform}.
In particular, these methods learn a flow-based generative model using flow matching \citep{lipman2023flow,liu2023flow,albergo2022building} to imitate the behavior policy, resulting in a compact latent space that maps to the behavior policy.
Assuming the decoder maps the latent prior's support into the behavior policy's support, any latent policy restricted to that support produces in-support actions, while allowing their probabilities to differ from those of the behavior policy.
Policy optimization performed on this latent space, either via gradient descent (as in \citep{zhang2026reform}) or SAC \citep{haarnoja2018soft} (as in \citep{wagenmaker2025steering}), yields a policy that satisfies support constraints by construction.

In this work, we argue that support constraints are only part of the story.
Once policy optimization is restricted to the support, the central difficulty becomes policy extraction: how should one reliably select high-value actions from the support without collapsing to brittle modes or exploiting sharp artifacts of the learned critic?
Even when actions remain in support, the critic may be sharp, irregular, or locally overestimated, and a low-entropy extraction procedure can collapse onto brittle modes.

We study \textbf{entropy regularization} as an explicit objective for support-constrained flow policy optimization.
While entropy regularization is well studied in online RL for its role in improving exploration \citep{ziebart2010modeling,haarnoja2018soft} and robustness \citep{ziebart2010modeling,eysenbach2022maximum}, its role in offline RL, and in particular in support-constrained policy extraction, has not been isolated and made explicit.
We argue that entropy regularization is also crucial in offline RL, not only in discouraging premature collapse of the policy to local minima, but also in smoothing the policy optimization landscape by regularizing the critic \citep{ahmed2019understanding}.

To this end, we propose \textbf{\texttt{L}}atent Space \textbf{\texttt{A}}djoint Matching for \textbf{\texttt{S}}upport-Constrained \textbf{\texttt{E}}ntropy-\textbf{\texttt{R}}egularized Offline RL (\textbf{\texttt{LASER}}),
a framework for entropy-regularized support-constrained policy extraction with latent-space flow policies.
\algo{} builds on the latent support parameterization used in recent flow-based offline RL methods \citep{zhang2026reform,wagenmaker2025steering},
but replaces implicit or architecture-dependent regularization with an explicit entropy-regularized policy improvement objective.
The main technical challenge is that computing and optimizing the entropy of a flow policy in latent space is nontrivial.
We address this using \textbf{adjoint matching} \citep{domingo2024adjoint},
which enables entropy-regularized optimization of expressive flow policies \textit{in latent space} while preserving the support constraints induced by the learned latent space.

We make the following contributions:
\begin{enumerate}[leftmargin=2em]
    \item We identify entropy regularization as a central ingredient for stable support-constrained policy extraction with flow-matching policies in offline RL.
    \item We reinterpret recent support-constrained flow policy methods, including DSRL and ReFORM, through the lens of entropy regularization.
    \item We introduce \algo{}, an algorithm for entropy-regularized offline policy improvement with latent-space flow policies that preserves dataset-support constraints while enabling smoother and more reliable policy extraction.
    \item Extensive experiments demonstrate that \algo{}, \textit{using a fixed set of method-specific hyperparameters}, achieves the strongest performance profiles among the evaluated flow-policy baselines.
\end{enumerate}

\vspace{-1ex}

\section{Preliminaries}

\subsection{Offline RL}

We use $\Delta(\mathcal X)$ to denote the set of probability distributions over $\mathcal X$, and \pl{gray} notation for placeholder variables.
A Markov Decision Process (MDP) is defined by $\mathcal M = (\mathcal S, \mathcal A, r, \rho_0, P, \gamma)$, with state space $\mathcal S$, action space $\mathcal A$, reward function $r(\pl{s}, \pl{a}): \mathcal S\times \ActionSpace{} \to\mathbb R$, initial state distribution $\rho_0\in\Delta(\mathcal S)$, transition dynamics function $P(\pl{s'}|\pl{s}, \pl{a}):\mathcal S\times\mathcal A\to\Delta(\mathcal S)$, and discount factor $\gamma\in[0, 1]$. Given a dataset of trajectories $\mathcal D$ generated by some behavior policy $\pi_\beta(\pl{a}|\pl{s}):\mathcal S\to\Delta(\mathcal A)$, offline RL aims to learn a policy
$\FullPolicy{}(\pl{a}|\pl{s}) : \mathcal S\to\Delta(\mathcal A)$
that maximizes the expected discounted return $R(\FullPolicy{}) = \mathbb E_{\tau \sim \rho^{\FullPolicy{}}(\pl{\tau})}\left[\sum_{h = 0}^H\gamma^h r(s_h, a_h)\right]$, where $\rho^{\FullPolicy{}}(\tau) = \rho_0(s_0)\FullPolicy{}(a_0|s_0)P(s_1|s_0,a_0)\cdots\FullPolicy{}(a_H|s_H)$. Note that sampling trajectories during training is not allowed in the offline RL setting.

\textbf{The OOD Problem.}
A primary challenge in offline RL is handling OOD actions \citep{levine2020offline}. Standard actor-critic frameworks evaluate a policy $\FullPolicy{}$ by estimating its state-action value function ($Q$-function) $Q(\pl{s}, \pl{a}):\mathcal S\times\mathcal A\to\mathbb R$, representing the expected discounted return:
\begin{equation} \label{eq: q-function}
    Q^{\FullPolicy}(s, a) = \mathbb{E} \left[\, \sum_{h=0}^H \gamma^h r(s_h, a_h) \biggm\vert s_0 = s, a_0 = a, a_h \sim \FullPolicy(s_h), \forall h \geq 1 \right].
\end{equation}
To learn $Q^{\FullPolicy}$,
fitted Q-evaluation is typically used to minimize the temporal difference (TD) error:
\begin{equation}\label{eq: critic-loss}
    L(\phi) = \mathbb E_{(s, a, s')\sim\mathcal D, a'\sim \FullPolicy(s')}\left[ \left( r(s, a) + \gamma Q_{\hat\phi}^{\FullPolicy}(s', a') - Q_\phi^{\FullPolicy}(s, a) \right)^2 \right],
\end{equation}
where $Q_{\hat{\phi}}^{\FullPolicy}$ is a target network often updated via Polyak averaging \citep{polyak1992acceleration}. However, because the dataset $\mathcal{D}$ is fixed, querying the target network with unfamiliar actions $a' \sim \FullPolicy(s')$ can yield unreliable and potentially large $Q$-values. Consequently, the policy may erroneously optimize toward these OOD actions \citep{levine2020offline}.
To address this issue, the learned policy should satisfy the \textbf{support constraint} \citep{kumar2019stabilizing}:
\begin{equation}\label{eq: support-constraint}
    \supp(\FullPolicy(\cdot|s))\subseteq\supp(\BehaviorPol{}(\cdot|s)), \quad \forall s,
\end{equation}
where $\supp(p) \coloneqq \{ x \mid p(x) > 0 \}$ denotes the support of a distribution $p$.
To achieve this, many prior offline RL algorithms regularize the statistical distance $D(\FullPolicy\;\|\;\pi_\beta)$ between the learned policy $\FullPolicy$ and the behavior policy $\pi_\beta$, resulting in the following objective for policy updates:
\begin{equation}\label{eq: statistical-distance-reg-loss}
    \min_{\FullPolicy} \mathbb E_{s\sim\mathcal D, a\sim\FullPolicy(s)}\left[-Q_\phi^{\FullPolicy}(s, a) + \alpha D(\FullPolicy \;\|\; \pi_\beta)\right],
\end{equation}
where $\alpha$ is a hyperparameter for the regularization weight.
However, depending on the statistical distance $D$, these approaches may fail to prevent OOD actions or may limit the improvement of $\FullPolicy$ over the behavior policy $\pi_\beta$ \citep{zhang2026reform}. The weight $\alpha$ also requires task- and dataset-specific tuning \citep{levine2020offline,ogbench_park2025}.

\subsection{Mitigating OOD Actions with Latent Space RL} \label{sec: latent-space-rl}
To mitigate the limitations of statistical distance regularization, recent works propose learning policies in a \textit{latent space} $\LatentSpace{}$ \citep{zhou2021plas,wagenmaker2025steering,zhang2026reform}.
Specifically, a \textbf{decoder} $\mu_{\thetaLToA}^{\pl{s}}(\pl{z}) : \LatentSpace{}\times\mathcal S \to \ActionSpace{}$ is learned that maps samples $z \sim \LatentBCDist{}$ drawn from the latent-space prior distribution $\LatentBCDist{}(\pl z) \in \Delta(\LatentSpace{})$ to the dataset distribution $\BehaviorPol(\cdot | s)$,
i.e., the pushforward of the prior $\LatentBCDist{}$ through the decoder $\muLToA$ approximates the behavior policy $\BehaviorPol(\cdot | s)$.
Note that the decoder is \textit{deterministic} for a given state and latent input, and the stochasticity of the action-space policy comes from sampling the latent input.
Then, instead of learning a policy $\FullPolicy(\cdot | s)$ in the action space $\ActionSpace{}$,
we can learn a \textbf{latent-space policy} $\LatentPolicy{}(\pl z | \pl{s}) : \mathcal{S} \to \Delta(\LatentSpace{})$ which induces a full policy in the action space via the pushforward through the decoder $\muLToA$, i.e.,
\begin{equation} \label{eq: full-policy-pushforward}
    \FullPolicy{}(\cdot | s) \coloneqq (\muLToACondS)_\sharp \LatentPolicy{}(\cdot | s),
\end{equation}
where $(\cdot)_\sharp$ is the pushforward operator.
This parameterization {\eqref{eq: full-policy-pushforward}} transforms the difficult action-space support constraint
{\eqref{eq: support-constraint}} under the complex behavior policy $\BehaviorPol{}$
into a simple latent-space constraint:
\begin{equation} \label{eq: latent-support-constraint}
    \supp(\LatentPolicy{}(\cdot|s)) \subseteq \supp(\LatentBCDist{}).
\end{equation}
Since $\LatentBCDist{}$ is chosen to have simple support (e.g., a uniform distribution on a $d$-ball \citep{zhang2026reform}),
this constraint is much easier to enforce, thereby decoupling OOD action avoidance from policy optimization.

We can thus define the resulting policy optimization problem in the latent space $\LatentSpace{}$ induced by the decoder $\muLToA$.
For a $Q$-function $Q_\phi^{\FullPolicy{}}$ in the action space \eqref{eq: q-function},
this induces a corresponding \textbf{latent-space $Q$-function} $Q_\phi^{\LatentPolicy{}}(\pl{s}, \pl{z}) : \mathcal{S} \times \LatentSpace \to \mathbb{R}$ as the pullback of $Q_\phi^{\FullPolicy{}}$ through the decoder, i.e.,
\begin{equation} \label{eq: latent-q-function}
    Q_\phi^{\LatentPolicy{}}(s, z) = Q_\phi^{\FullPolicy{}}(s, \muLToA(z; s)).
\end{equation}
As long as the latent-space support constraint {\eqref{eq: latent-support-constraint}} is satisfied (e.g., with $\tanh$ squashing, clipping or projection), \textit{any} off-policy optimization algorithm, such as DDPG or TD3 in {\citep{zhou2021plas}}, SAC in {\citep{wagenmaker2025steering}}, or even plain gradient ascent in {\citep{zhang2026reform}}, can then be used to optimize the latent policy $\LatentPolicy{}$ \textit{without} needing to worry about OOD actions.

\begin{figure}[t]
    \centering
    \stepcounter{figure}
    \setcounter{subfigure}{0}
    \vspace{-1em}
    \begin{overpic}[width=1.02\linewidth]{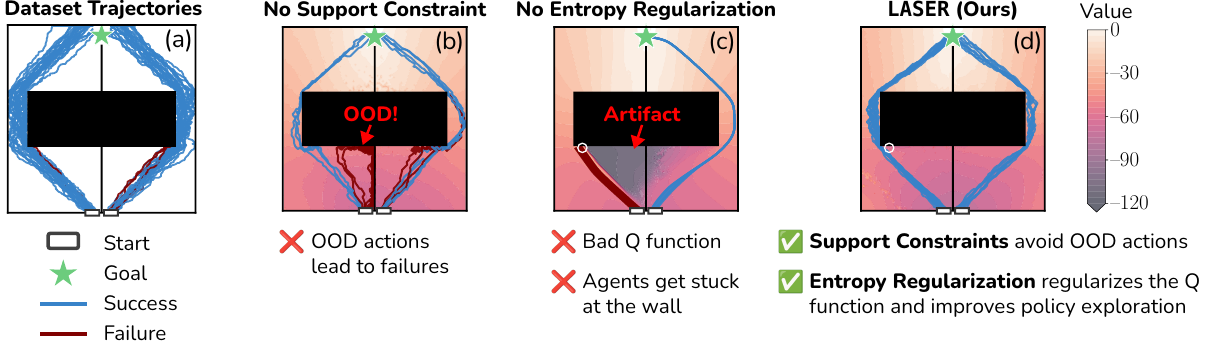}
        \put(1, 24){\refstepcounter{subfigure}\label{fig: point-pos:dataset}}
        \put(24, 24){\refstepcounter{subfigure}\label{fig: point-pos:no_sup}}
        \put(46, 24){\refstepcounter{subfigure}\label{fig: point-pos:no_ent}}
        \put(76, 24){\refstepcounter{subfigure}\label{fig: point-pos:ours}}
    \end{overpic}
    \addtocounter{figure}{-1}
    \vspace{-1.5em}
    \caption{\textbf{Both support constraints and entropy regularization are essential.}
    Without support constraints, the policy generates OOD actions that have unreliable $Q$-values and thus fail.
    Without entropy regularization, the $Q$-function exhibits strange artifacts and the agents get stuck at the wall, resulting in failures on the left side.
    By combining support constraints and entropy regularization, \texttt{LASER} avoids OOD actions, regularizes the $Q$-function, and encourages policy exploration, resulting in a successful policy.
    }
    \label{fig: point-pos}
\end{figure}
\begin{figure}
    \centering
    \stepcounter{figure}
    \setcounter{subfigure}{0}
    \begin{overpic}[width=1.02\linewidth]{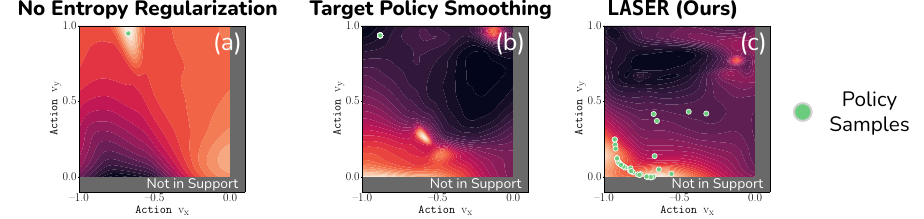}
        \put(1, 24){\refstepcounter{subfigure}\label{fig: point-pos:q-fn:no_ent}}
        \put(24, 24){\refstepcounter{subfigure}\label{fig: point-pos:q-fn:tps}}
        \put(76, 24){\refstepcounter{subfigure}\label{fig: point-pos:q-fn:ours}}
    \end{overpic}
    \addtocounter{figure}{-1}
    \vspace{-1.5em}
    \caption{
        \textbf{$\mathbf{Q}$-function and policy are poorly learned without entropy regularization.}
        We plot the learned $Q$-functions (lighter colors denote higher value) at the state marked by the white circle in \Cref{fig: point-pos:no_ent}.
        Without entropy regularization, the policy becomes nearly deterministic and clusters at a \textit{false} maximum which causes the agents to move up and get stuck at the wall (\Cref{fig: point-pos:no_ent}).
        Adding noise to the policy in TD learning (target policy smoothing \citep{fujimoto2018addressing}) fixes the $Q$-function landscape, but the policy remains nearly deterministic and is suboptimal.
        With entropy regularization, the policy samples spread out more and yield a better regularized $Q$-function, enabling the agents to move around the wall and reach the target (\Cref{fig: point-pos:ours}).
        }
    \label{fig: point-pos:q-fn}
\end{figure}

\textbf{Learning expressive decoders with flow matching.}
A key component of latent-space RL is the decoder $\muLToA$.
A poorly learned decoder may output actions outside the dataset support and defeat the purpose of latent-space RL, or be too simple to capture the behavior policy and limit the performance of the final policy.
Recent works use flow matching \citep{lipman2023flow,liu2023flow,albergo2022building} to learn $\muLToA$ due to its ability to model complex multimodal behavior distributions \citep{wagenmaker2025steering,zhang2026reform}.
Flow matching learns a time-dependent velocity field $v_\thetaLToA(\pl z, \pl t; \pl s): \LatentSpace{}\times[0, 1]\times\mathcal S\to\LatentSpace{}$, parameterized by $\thetaLToA$, that transports an easy-to-sample latent distribution $\LatentBCDist{}\in\Delta(\mathcal L)$ (e.g., standard Gaussian $\Gaussian(0, I^d)$) to the pushforward action distribution
$(\muLToACondS)_\sharp \LatentBCDist{} \approx \BehaviorPol(\cdot|s)$
over the time interval $t \in [0, 1]$,
using the flow-matching loss:
\begin{equation}\label{eq: flow-matching}
    L_\LToA(\thetaLToA) = \mathbb E_{(s, a)\sim\mathcal D, z\sim \LatentBCDist{}, t\sim \mathcal U[0, 1]}\left[ \norm{v_{\thetaLToA}(x_t, t; s) - (a - z)}^2 \right], \quad x_t = (1 - t)z + ta.
\end{equation}
The corresponding BC \textbf{flow map} $\psi_{\thetaLToA}(\pl z, \pl t; \pl s): \LatentSpace{} \times [0, 1] \times \mathcal S \to \LatentSpace{}$ for the velocity field $v_{\thetaLToA}$ can be found as the solution to the ODE with initial condition $z$ at $t=0$ and velocity field $v_{\thetaLToA}$, i.e.,
\begin{equation} \label{eq: bc-ode}
    \psi_{\thetaLToA}(z, t; s) = z + \int_0^t v_{\thetaLToA}(\psi_{\thetaLToA}(z, \tau; s), \tau; s) \diff \tau.
\end{equation}
The decoder $\muLToA$ is then defined as the BC flow map at time $t=1$, i.e., $\muLToA(z; s) \coloneqq \psi_{\thetaLToA}(z, 1; s)$.

\section{Latent Space Adjoint Matching for Support-Constrained Entropy-Regularized Offline RL (\textbf{\texttt{LASER}})}
\subsection{Entropy Regularization is Essential for Latent Space RL}
\label{sec: er-essential}
Given that latent-space RL methods claim to solve the OOD problem,
one might be tempted to conclude that directly performing policy optimization in the latent space with support constraints results in a strong policy.
However, we find this to be false.
We illustrate this in a simple 2D reach-avoid environment (see \Cref{app: point-pos-detail} for details).
Choosing $\LatentPolicy{}$ to be a flow policy as in \citep{zhang2026reform}, i.e.,
for a base distribution $\BaseDist{}$ and velocity field $v_{\thetaNToL}$, and corresponding flow map $\psi_{\thetaNToL}$,
\begin{equation} \label{eq: latent-flow-policy}
    \LatentPolicy{}(\cdot | s) = (\muNToLCondS)_\sharp \BaseDist{}, \quad \muNToLCondS(w) \coloneqq \psi_{\thetaNToL}^{{s}}(w, 1),
\end{equation}
we find that na\"{\i}vely training $\LatentPolicy{}$ to maximize $Q_\phi^{\LatentPolicy{}}$ with reparametrized gradients and enforcing the latent-space support constraint \eqref{eq: latent-support-constraint} via $\tanh$ squashing
(\Cref{fig: point-pos:no_ent}) yields a policy that causes the agent to become stuck at the wall and fail to reach the target (\Cref{fig: point-pos}).
Moreover, the value function is poorly learned and exhibits artifacts.
We explore this further by visualizing the learned $Q$-function
and policy $\FullPolicy{}$
at the state right before the agent gets stuck at the wall (\Cref{fig: point-pos:q-fn:no_ent}).
Surprisingly, the learned policy is maximizing the learned $Q$-function and is not stuck in a local optimum despite the policy failing to reach the target.
Using target policy smoothing (TPS) \citep{fujimoto2018addressing} to smooth the learned $Q$-function by adding noise to the target action $a'$ in the TD learning target \eqref{eq: critic-loss}
fixes the $Q$-function landscape (\Cref{fig: point-pos:q-fn:tps}),
but the resulting policy is now suboptimal.
This is a problem of the policy having insufficient exploration\footnote{Here, exploration refers to sampling diverse in-support actions during policy optimization and critic updates.}, which can be mitigated by increasing policy entropy \citep{haarnoja2018soft}.

This motivates the use of an \textbf{explicit entropy regularization term} to obtain a policy with higher entropy, which not only helps with policy optimization via exploration, but also obviates the need for TPS in TD learning and yields a better regularized $Q$-function.
Namely,
for weight $\alpha > 0$,
we consider the following entropy-regularized objective to train the latent-space policy $\LatentPolicy{}$:
\begin{equation} \label{eq: entropy-reg-pol-loss}
    L_{\NToL}(\thetaNToL) = -\mathbb E_{s\sim\mathcal D}\left[\mathbb E_{w\sim \BaseDist{}}\left[Q_{{\phi}}^{\LatentPolicy{}}(s, \muNToL(w; s))\right] + \alpha \mathcal H (\muNToL(\cdot; s))\right].
\end{equation}
The entropy term can also be optimized using reparametrized gradients by computing the log probability under the flow map $\muNToL$ using the continuity equation \citep{lipman2023flow} to obtain
\begin{equation} \label{eq: flow-entropy}
    \mathcal H(\mu_{\thetaNToL}(\cdot; s)) = \mathbb E_{w \sim \BaseDist{}}\left[-\log \BaseDist{}(w) + \int_0^1 \nabla_{x_t} \cdot v_{\thetaNToL}(x_t, t; s) \diff t \right], \quad x_t = \psi_{\thetaNToL}(w, t; s).
\end{equation}
We discuss a more efficient method that does not require backpropagation through the flow map in the next subsection.
Adding explicit entropy regularization improves both the $Q$-function and the policy (\Cref{fig: point-pos:q-fn:ours}),
resulting in a policy that successfully guides the agent around the wall to the target in the simple illustrative environment (\Cref{fig: point-pos:ours}).

\begin{remark}
Examining recent latent-space RL methods, we find that their policy extraction does include an entropy regularization term despite its importance not being made explicit in their papers.
This is done either explicitly,
as in DSRL \citep{wagenmaker2025steering}, which uses a Gaussian latent policy and an SAC-style objective \citep{haarnoja2018soft},
or implicitly as in ReFORM \citep{zhang2026reform}, where the geometric reflection term prevents latent collapse (see \Cref{app: reform-issue}).
However, using an expressive flow policy for the latent-space policy and explicitly considering entropy regularization enables our method to empirically Pareto-dominate Gaussian policies (\Cref{fig: ent-q-tradeoff}) and maximize the $Q$-function better than other implicit entropy regularization schemes (\Cref{app: reform-issue}).
A comparison of methods with no, implicit, and explicit entropy regularization across 20 tasks in more complex manipulation environments further supports the benefit of explicit regularization (\Cref{app: experiment-additional-results}).
\end{remark}

\begin{figure}
    \centering
    \includegraphics[width=.7\linewidth]{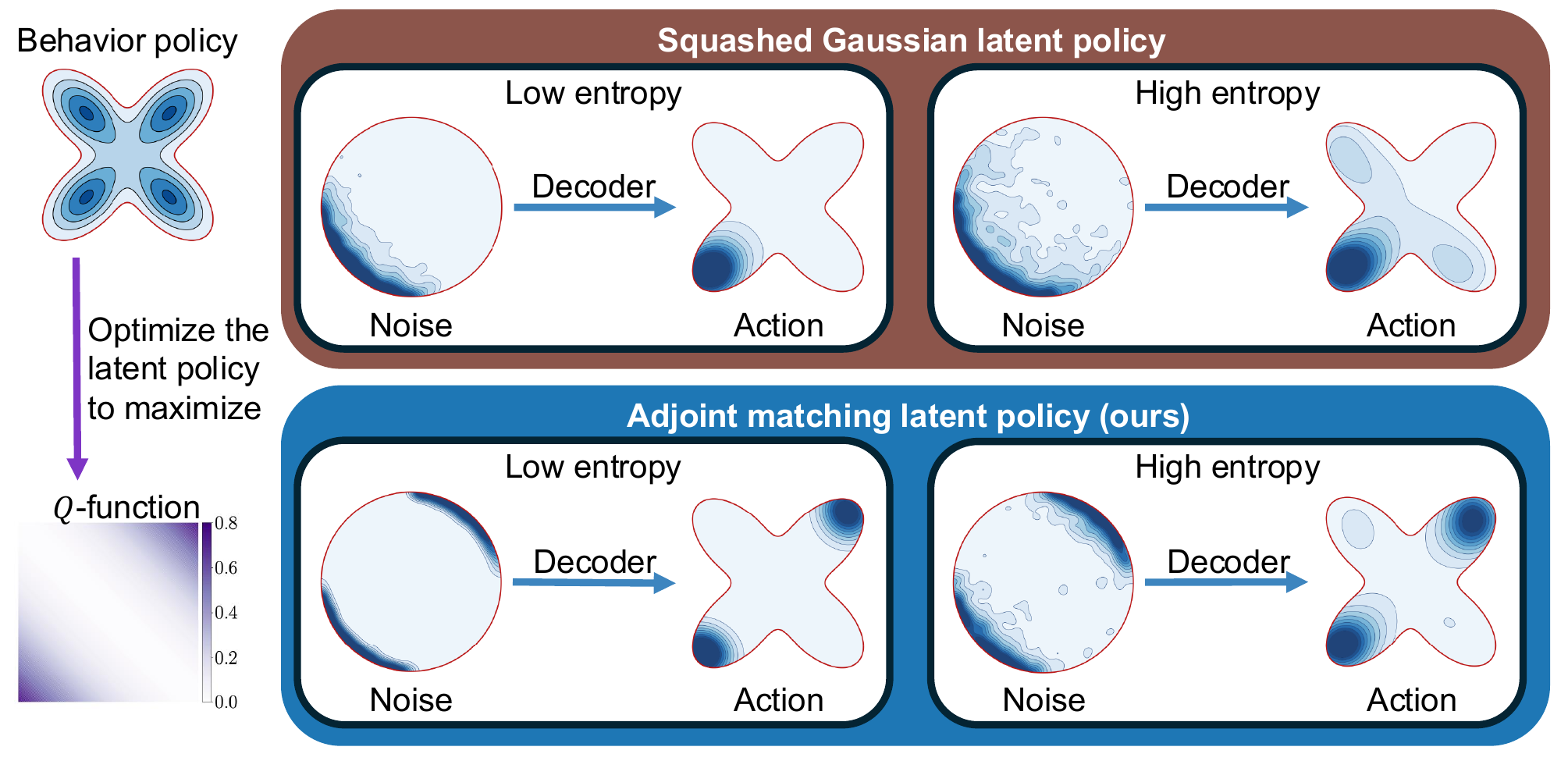}
    \includegraphics[width=.29\linewidth]{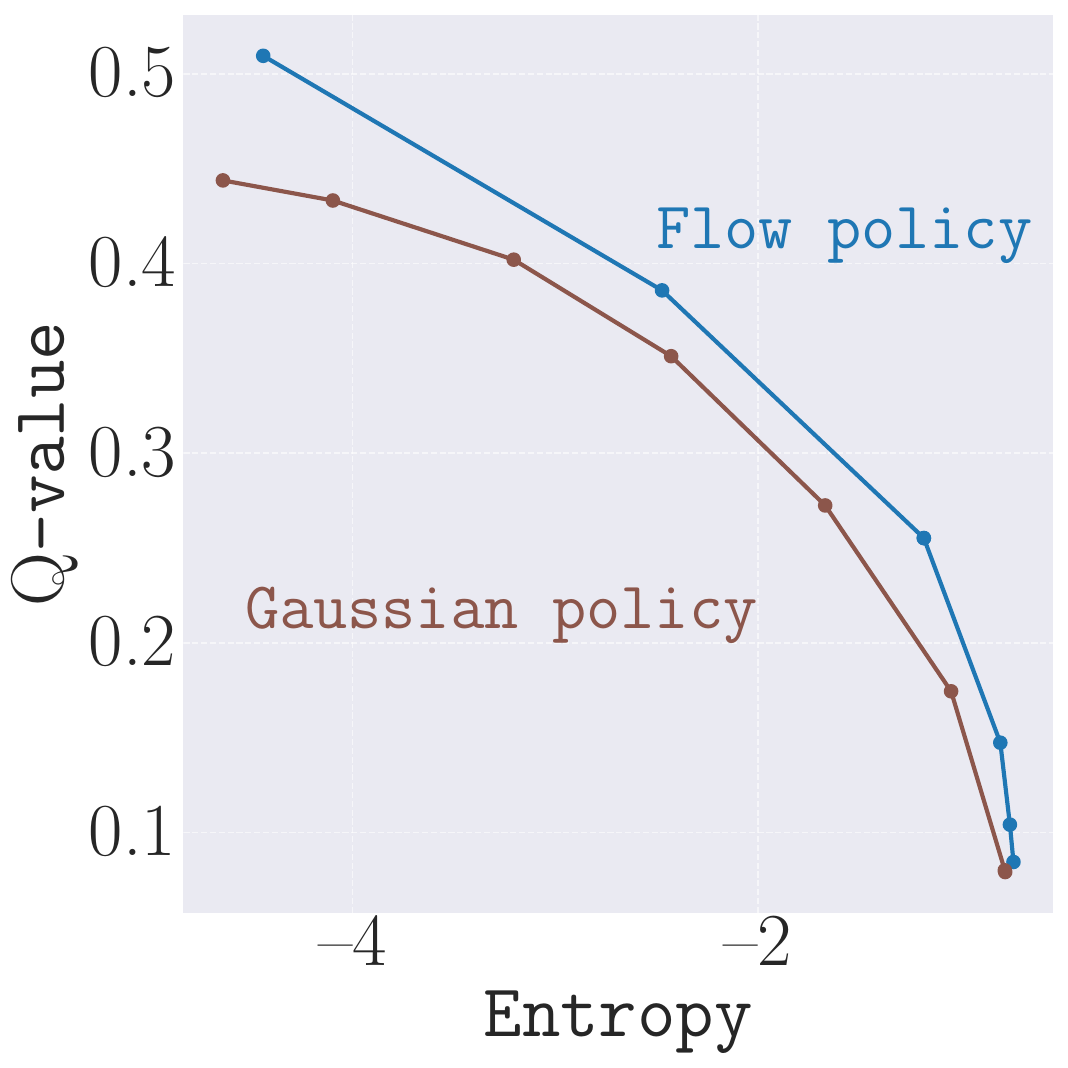}
    \caption{\textbf{Flow latent policy + adjoint matching Pareto-dominates a squashed Gaussian latent policy.}
       For every entropy and $Q$-value achieved by the Gaussian latent policy, there exists an operating point achievable by the flow latent policy that has both higher entropy and higher $Q$-value.
    }
    \label{fig: ent-q-tradeoff}
\end{figure}

\textbf{Why use a flow policy for $\mathbf{\pi}_\texttt{L}$?}
While a Gaussian latent-space policy significantly simplifies entropy-regularized optimization,
we motivate our use of a flow model with two arguments.
(1) A Gaussian $\LatentPolicy{}$ produces unimodal latents, which restricts the ability of the full policy in action space to capture multimodality \citep{zhang2026reform}.
(2) While high-entropy policies facilitate robust $Q$-function learning and help avoid local minima, they often yield lower $Q$-values. By being more expressive, flow policies empirically Pareto-dominate Gaussian policies on the performance-entropy tradeoff (\Cref{fig: ent-q-tradeoff}).

\subsection{Efficient Entropy-Regularized Policy Optimization via Adjoint Matching}\label{sec: adjoint-matching}

{Efficient entropy estimation for expressive policies remains challenging \citep{messaoud2024s2ac,messaoud2026particles}.}
While \eqref{eq: entropy-reg-pol-loss} can be used as-is to train the latent policy $\LatentPolicy{}$ with reparametrized gradients,
computing the gradient {for this flow parameterization} requires backpropagating through the flow map $\muNToL$ \eqref{eq: latent-flow-policy} and the entropy \eqref{eq: flow-entropy}, both of which are solutions to ODEs and require backpropagation through time (BPTT), which is computationally expensive and prone to numerical instability \citep{fql_park2025,sorl,li2026qam}.

To circumvent these computational bottlenecks, we adopt an alternative approach.
Instead of optimizing the loss directly via BPTT,
we derive the optimal latent-space policy $\LatentPolicy{}^\star(\cdot | s) = (\mu_{\thetaNToL}^s)^\star_\sharp \BaseDist{}$ that minimizes the entropy-regularized policy loss \eqref{eq: entropy-reg-pol-loss} in the typical case where the support of the latent-space prior distribution $\LatentBCDist{}$ has finite Lebesgue measure (e.g., a $d$-ball).

\begin{lemma}\label[lemma]{thm: ng-target-optimal}
    Fix a state $s \in \mathcal{S}$ and parameters of the critic and the decoder.
    Suppose $\supp(\LatentBCDist{})$ has finite positive Lebesgue measure, and the $Q$-function $Q^{\LatentPolicy{}}_\phi(s, \cdot)$ is measurable and has a finite upper bound on $\supp(\LatentBCDist{})$.
    Then, among latent densities supported on $\supp(\LatentBCDist{})$, the optimal distribution $\LatentPolicy{}^*$ that minimizes the loss \eqref{eq: entropy-reg-pol-loss} is given by the following $Q$-tilted distribution:
    \begin{equation}\label{eq: ng-target-optimal}
        \LatentPolicy{}^\star(\cdot|s) \propto \LatentUniform{}(\cdot)\exp\left( Q^{\LatentPolicy{}}_\phi(s, \cdot) / \alpha \right),
    \end{equation}
    where $\LatentUniform{} = \mathcal U( \supp(\LatentBCDist{}))$ is the uniform distribution on the support of $\LatentBCDist{}$.
\end{lemma}
This optimal latent distribution has the Gibbs form used in maximum-entropy RL \citep{haarnoja2018soft,messaoud2024s2ac}, here restricted to the latent support.
See \Cref{app: ng-target-optimal:proof} for the proof.
Note that this is a statewise, fixed-critic optimum and it does not establish convergence of the joint actor-critic updates.
With \Cref{thm: ng-target-optimal}, we can optimize \eqref{eq: entropy-reg-pol-loss} by learning to sample from the optimal distribution $\LatentPolicy{}^\star$, which can be done efficiently \textit{without BPTT} using \textbf{adjoint matching} \citep{domingo2024adjoint}.

Namely, suppose both $\LatentPolicy{}$ and $\LatentUniform{}$ are flow models with a common base distribution $\BaseDist{}$, i.e., \eqref{eq: latent-flow-policy} holds for $\LatentPolicy{}$, and for flow map $\psiNToLBase$ generated by velocity field $\vNToLBase$,
\begin{equation}
    \LatentUniform{} = (\muNToLBase)_\sharp \BaseDist{},\quad \muNToLBase(w) = \psiNToLBase(w, 1),
\end{equation}
and let $p_v(\{ \pl{w}_t\}_{t \in [0, 1]}) \in \Delta(\NoiseSpace^{[0, 1]})$ denote the trajectory law of the following memoryless stochastic differential equation (SDE) with control $v$ and $w_0\sim\BaseDist{}$:
\begin{equation}\label{eq: forward-sde}
    \diff w_t = \left( 2\, v(w_t, t; s) - w_t / t \right) dt + \sigma_t \diff B_t,
\end{equation}
where $\sigma_t=\sigma_0\sqrt{2(1-t)/t}$ is the memoryless noise schedule and $B_t$ is standard Brownian motion.
We fix the state $s$ and suppress it in $p_v$. For the Gaussian base $\BaseDist{}=\mathcal N(0,\sigma_0^2I)$ used here, the reference SDE with $v=\vNToLBase$ shares the time marginals of the reference flow ODE in the exact flow-matching formulation, including the terminal marginal $\LatentUniform{}$.
Moreover, consider the tilted distribution over \textit{SDE trajectories} $p_{v^\star}$,
defined as
\begin{equation} \label{eq: am-target-dist}
    p_{v^\star}(\{w_t\}_{t \in [0, 1]}) \propto p_{\vNToLBase}(\{w_t\}_{t \in [0, 1]}) \exp\left( Q^{\LatentPolicy{}}_\phi(s, w_1) / \alpha \right),
\end{equation}
where \eqref{eq: ng-target-optimal} now corresponds to the marginal of $p_{v^\star}$ at $t=1$.
Then, $p_{v^\star}$ can be learned without BPTT
by repeatedly optimizing a regression loss using a ``special estimate'' of $v^\star$ as the regression target.
Specifically, the optimal control $v^\star$ realizing \eqref{eq: am-target-dist} solves the following stochastic optimal control (SOC) problem with quadratic control cost and terminal cost given by the negative $Q$-function \citep{domingo2024adjoint}:
\begin{equation}
    \min_v \mathbb{E}_{s \sim \mathcal{D}, \{w_t\}} \left[ \int_0^1 \frac{2}{\sigma_t^2} \left\| v(w_t, t; s) - \vNToLBase(w_t, t) \right\|^2 dt - Q_{\phi}^{\LatentPolicy{}}(s, w_1) / \alpha \right].
\end{equation}
Then, given a rollout $\{w_t\}_{t \in [0, 1]}$ sampled with control $v$, we can construct the following estimate $\hat{v}^\star$ of the optimal control $v^\star$:
\begin{equation}\label{eq: v-star}
    \hat{v}^\star(w_t, t; s) = \vNToLBase(w_t, t) - \sigma_t^2 \tilde{g}_t / 2,
\end{equation}
where $\tilde{g}_t$ is the ``lean'' adjoint state computed by integrating the reverse ODE using $\{w_t\}_{t \in [0, 1]}$
\begin{equation} \label{eq: lean-adjoint}
    d\tilde{g}_t = - \left(\nabla_{w_t} \left[ 2\vNToLBase(w_t, t) - w_t / t  \right]\right)^{\top} \tilde{g}_t dt, \quad \tilde{g}_1 = - \nabla_{w_1} Q^{\LatentPolicy{}}_\phi(s, w_1) / \alpha.
\end{equation}
{In implementation, we substitute the target critic $Q_{\hat\phi}^{\LatentPolicy{}}$ for $Q_\phi^{\LatentPolicy{}}$ in these policy-improvement expressions (\Cref{app: implementation}).}
Using the estimate $\hat{v}^\star$ as the regression target in the following regression loss then results in a unique fixed point at the optimal control $v^\star$ \citep{domingo2024adjoint}, where we replace the current velocity field $v$ with the trainable velocity field $v_{\thetaNToL}$ to be optimized:
\begin{equation}\label{eq: am}
    {L}_\mathrm{AM}(\thetaNToL) = \mathbb{E}_{s \sim \mathcal{D}, \{w_t\}} \left[ \frac{1}{2} \int_0^1 \norm[\Big]{ \frac{2}{\sigma_t} \Big( v_{\thetaNToL}(w_t, t; s) - \hat{v}^\star(w_t, t; s) \Big) }^2 dt \right].
\end{equation}
Importantly, optimizing \eqref{eq: am} \textit{does not require backpropagation through the flow} and only requires the solution of an ODE to compute the adjoint state $\tilde{g}_t$ for the regression target $\hat{v}^\star$, circumventing the computational and numerical stability issues of BPTT.
For a more comprehensive mathematical treatment of adjoint matching, we refer readers to \citep{domingo2024adjoint}.

\begin{figure}
    \centering
    \includegraphics[width=1.01\linewidth]{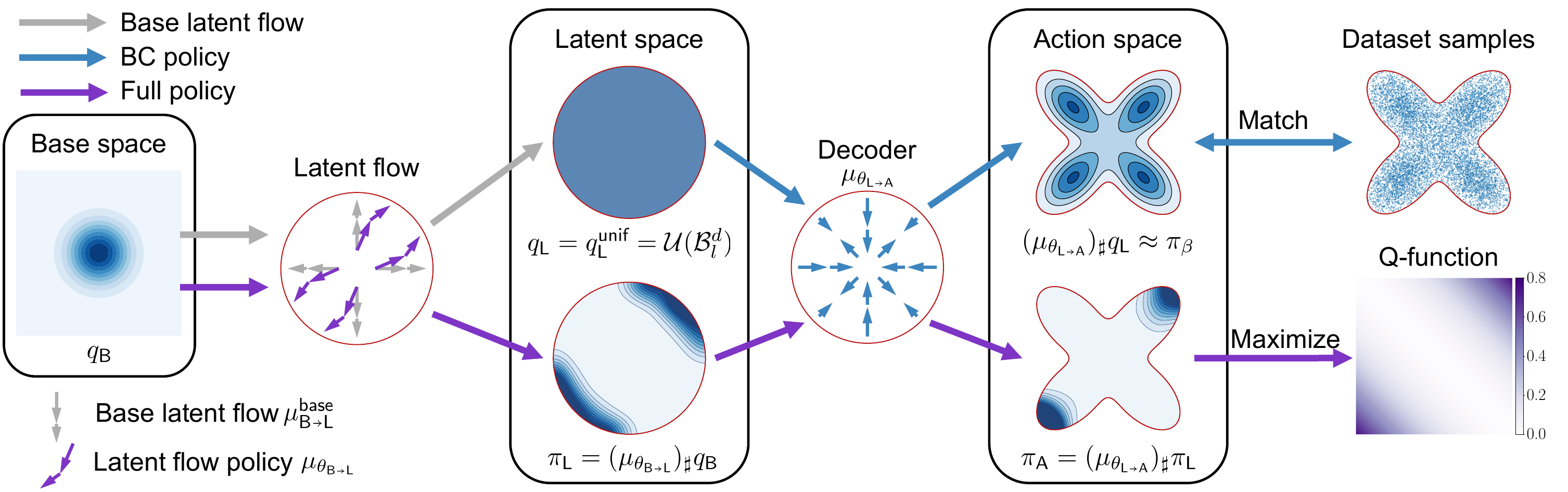}
    \caption{\textbf{\algo{} Algorithm.}
    (\bcarrow{}) We train a flow-matching decoder $\muLToA$
    that maps samples from the uniform distribution in the $d$-ball $\LatentBCDist{} = \LatentUniform{} = \mathcal U(\mathcal B_l^d)$ in \textit{latent space} to the dataset actions in \textit{action space}.
    (\ngarrow{}) We then learn a base latent flow $\muNToLBase{}$ that maps the prior distribution $\BaseDist{}$ in the base space to the uniform distribution in the latent space $\LatentUniform{}$.
    (\amarrow{}) The base latent flow is then fine-tuned via adjoint matching, which results in a latent flow policy $\muNToL{}$ that generates the final action distribution $\FullPolicy{} = (\muLToA)_\sharp (\muNToL)_\sharp \BaseDist{}$.
}
    \label{fig: algo}
\end{figure}

\subsection{Algorithm}
\label{sec: full-algo}
We now present \textbf{\texttt{L}}atent Space \textbf{\texttt{A}}djoint Matching for \textbf{\texttt{S}}upport-Constrained \textbf{\texttt{E}}ntropy-\textbf{\texttt{R}}egularized Offline RL (\textbf{\texttt{LASER}}).
\algo{} combines the aforementioned insights to perform latent-space RL with an expressive flow decoder and efficient entropy-regularized policy optimization via adjoint matching.
\begin{itemize}[leftmargin=*, topsep=0pt, parsep=0pt]
    \item Following \citep{zhang2026reform}, we choose the latent prior $\LatentBCDist{} = \Uniform(\mathcal{B}_l^d)$ to be the uniform distribution on the ball $\mathcal{B}_l^d$ of radius $l$ in $\mathbb{R}^d$, which has finite Lebesgue measure and satisfies the conditions of \Cref{thm: ng-target-optimal}.
    \item Adjoint matching requires $\LatentUniform$, the uniform distribution over the support of $\LatentBCDist{}$, to also be a flow model. While an analytical vector field can be derived since $\LatentUniform = \LatentBCDist{}$ is the uniform distribution over the ball, it involves special functions with no easy closed-form expression (see \Cref{app: gt-flow-gaussian-uniform}). Instead, we train a flow model via flow matching to match $\LatentUniform$ with the same base distribution $\BaseDist{}$ as $\LatentPolicy{}$ to simplify the implementation.
\end{itemize}
We now summarize \algo{} (\Cref{fig: algo}),
which consists of four components trained jointly:
\textbf{(1. Flow-matching decoder)} We train a state-conditioned flow decoder $\muLToA$ using flow matching to map samples from the uniform distribution on the $d$-ball $\LatentBCDist{}$ to dataset actions, i.e., behavior cloning (\Cref{sec: latent-space-rl}).
\textbf{(2. Flow matching for $\LatentUniformBold{}$)} We train a flow model $\muNToLBase$ to match the uniform distribution $\LatentUniform{}$ in latent space from $\BaseDist{}$ to use in adjoint matching.
\textbf{(3. Critic learning)} We learn a critic $Q^{\FullPolicy}_\phi$ using the standard TD learning objective \eqref{eq: critic-loss}.
\textbf{(4. Entropy-regularized policy optimization)} Finally, we perform entropy-regularized policy optimization in latent space to maximize the latent-space $Q$-function $Q^{\LatentPolicy{}}_\phi$ \eqref{eq: latent-q-function} using adjoint matching (\Cref{sec: adjoint-matching}).\footnote{In practice, we project the numerical latent flow output onto $\mathcal B_l^d$ before decoding. This support correction approximates the density-based formulation (\Cref{app: implementation}).}
We provide practical implementation details and a high-level algorithm for \algo{} in \Cref{app: implementation}.

\section{Related Work}
We discuss the most closely related work in this section and defer a broader discussion of offline RL and generative policies to \Cref{app: extended_related}.

\textbf{Distribution shift in offline RL. }
A central challenge in offline RL is distribution shift: learned value
functions can overestimate actions outside the dataset support, causing
policy improvement to select out-of-distribution actions with poor true
value. Prior methods address this issue by regularizing the policy toward
the behavior policy with statistical distances, learning pessimistic value
functions, or explicitly restricting policy improvement to supported
actions \citep{kumar2019stabilizing,kumar2020conservative,kostrikov2022offline,wu2022spot,mao2023str}.
We follow the support-constrained perspective, using an
expressive flow policy to enforce support by construction.

\textbf{Expressive generative policies for offline RL. }
Diffusion \citep{song2020score,ho2020denoising} and flow-matching \citep{lipman2023flow,albergo2022building,liu2023flow} models have recently been used as expressive
alternatives to deterministic or Gaussian policies, especially in robotic
imitation learning where action distributions are often multimodal
\citep{chi2023diffusion,wang2023diffusion}.
These policies have also been adapted to offline RL, where the goal is to
improve beyond the behavior policy using a learned critic. Existing methods
differ in how they use the critic: some use critic values to select,
reweight, or clone high-value actions \citep{chen2023offline,hansen2023idql,kang2023efficient,Zhang2025EnergyWeightedFlowMatching}, while others use critic gradients to
directly improve the generated actions
\citep{ding2024diffusion,wang2023diffusion,fql_park2025}.
Most closely related to our optimization mechanism,
\citep{li2026qam} fine-tunes flow policies using critic
gradients under a KL-regularized reward objective. Unlike these methods,
which rely on statistical regularization to control distribution shift,
\algo{} enforces support constraints through its latent action
parameterization.

\textbf{Support-constrained policy extraction with latent-space RL. }
Recent methods avoid OOD actions by learning a latent action
space whose decoder maps into the support of the behavior policy, then
performing policy optimization in this latent space
\citep{zhou2021plas,wagenmaker2025steering,zhang2026reform}.
Our method builds on this idea, but revisits the optimization problem that
remains once support has been controlled. In particular, we show that
entropy regularization is important for preventing policy collapse and
stabilizing policy extraction within the supported action set.

\textbf{Entropy regularization. }
Entropy regularization is a classical tool for smoothing policy
optimization, improving robustness, and encouraging exploration
\citep{ziebart2010modeling,haarnoja2018soft,ahmed2019understanding}, with extensions to expressive policies \citep{messaoud2024s2ac,messaoud2026particles}.
In offline RL,
though some works have noted its benefits for exploration during online fine-tuning \citep{ding2024diffusion,dong2025expo},
it receives much less attention, perhaps due to the implicit entropy regularization effect from behavior policy regularization techniques.
Our perspective is different: because support constraints are enforced by
the policy parameterization, entropy regularization can be used for its
standalone role in policy extraction.

\section{Experiments}

We conduct experiments to answer the following research questions.
\textbf{(Q1): }How does \algo{} perform against other offline RL algorithms with flow policies?
\textbf{(Q2): }How sensitive is \algo{} to hyperparameters?
\textbf{(Q3): }How does adjoint matching compare with BPTT?
\textbf{(Q4): }Does the latent space decouple entropy regularization from behavior regularization?
Details on implementation, baselines, environments, and additional results and ablations are in \Cref{app: experiments}.

\subsection{Setup}

\textbf{Environments.}
We evaluate our method on the OGBench offline RL benchmark \citep{ogbench_park2025}. The main evaluation suite spans four environments, covering both locomotion and manipulation domains, with five tasks per environment. Each task is evaluated on two datasets (\datasetname{clean} and \datasetname{noisy}), yielding a total of $4 \times 5 \times 2 = 40$ evaluation settings. We additionally evaluate harder tasks in \Cref{app: experiment-additional-results}.

\textbf{Baselines.}
Since recent studies have shown that flow-based policies outperform traditional alternatives \citep{fql_park2025,agrawalla2025floq,sorl,li2026qam},
we benchmark \ourname{} against state-of-the-art flow-based methods.
These include \baselinename{IFQL} \citep{hansen2023idql,fql_park2025},
which combines a BC flow policy with an IQL critic and best-of-$N$ action selection; \baselinename{FQL} \citep{fql_park2025}, which applies Wasserstein distance regularization; and \baselinename{QAM} \citep{li2026qam}, which uses adjoint matching to approximate the $Q$-value-weighted BC action distribution, alongside its extension \baselinename{QAM-E} \citep{li2026qam} that learns an additional residual policy. We also compare against latent-space RL methods, specifically \baselinename{DSRL} \citep{wagenmaker2025steering} with a Gaussian latent policy, and \baselinename{ReFORM} \citep{zhang2026reform} with a reflected flow latent policy.
The method-specific hyperparameters of \ourname{}, \baselinename{ReFORM}, and \baselinename{IFQL} remain fixed across all evaluated tasks, while those of the other baselines vary by environment and dataset (\Cref{tab: env-specific-unique-hyperparams}). All methods share the environment-dependent action chunk lengths and pessimistic backup weights (\Cref{tab: env-specific-common-hyperparams}).

\subsection{Results}

\textbf{(Q1): \algo{} achieves state-of-the-art performance using a single set of method-specific hyperparameters.}
In \Cref{fig: main-exp}, we present the performance profiles \citep{agarwal2021deep} aggregating the 40 main evaluation settings across both the \datasetname{clean} and \datasetname{noisy} datasets.
The profiles show that \ourname{} outperforms all evaluated baselines in aggregate. It achieves at least $95\%$ of the best mean success rate in $34/40$ settings and a mean success rate of $100\%$ in $13/40$ settings. A paired statistical comparison supports the aggregate advantage (global $p=1.35\times10^{-3}$), with details in \Cref{app: experiment-additional-results}.
Notably, on the \datasetname{noisy} datasets, \ourname{} and \baselinename{ReFORM} achieve the highest mean success rates, underscoring the advantage of enforcing support constraints, which does not restrict policy improvement.
Furthermore, \ourname{} significantly improves upon \baselinename{ReFORM} by using explicit entropy regularization and avoiding the static-flow behavior (\Cref{app: reform-issue,app: experiment-additional-results}).

\begin{figure}
    \centering
    \includegraphics[width=\linewidth]{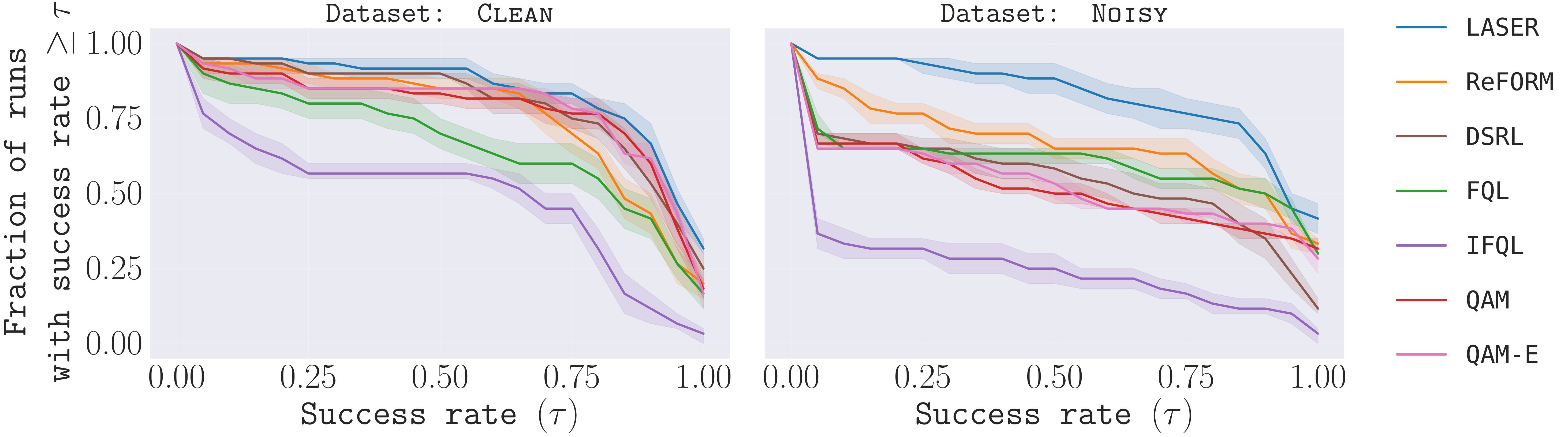}
    \caption{\textbf{Performance profile plots for the main evaluation suite with different datasets.} For a given success rate $\tau$ (x-axis), we plot the fraction of runs with a success rate $\geq\tau$ \citep{agarwal2021deep}. Shading shows pointwise $95\%$ bootstrap confidence intervals, computed by resampling training seeds within each task 2000 times. \ourname{} attains the strongest empirical performance profile on both datasets.}
    \label{fig: main-exp}
\end{figure}

\begin{figure}
    \centering
    \includegraphics[width=.32\linewidth]{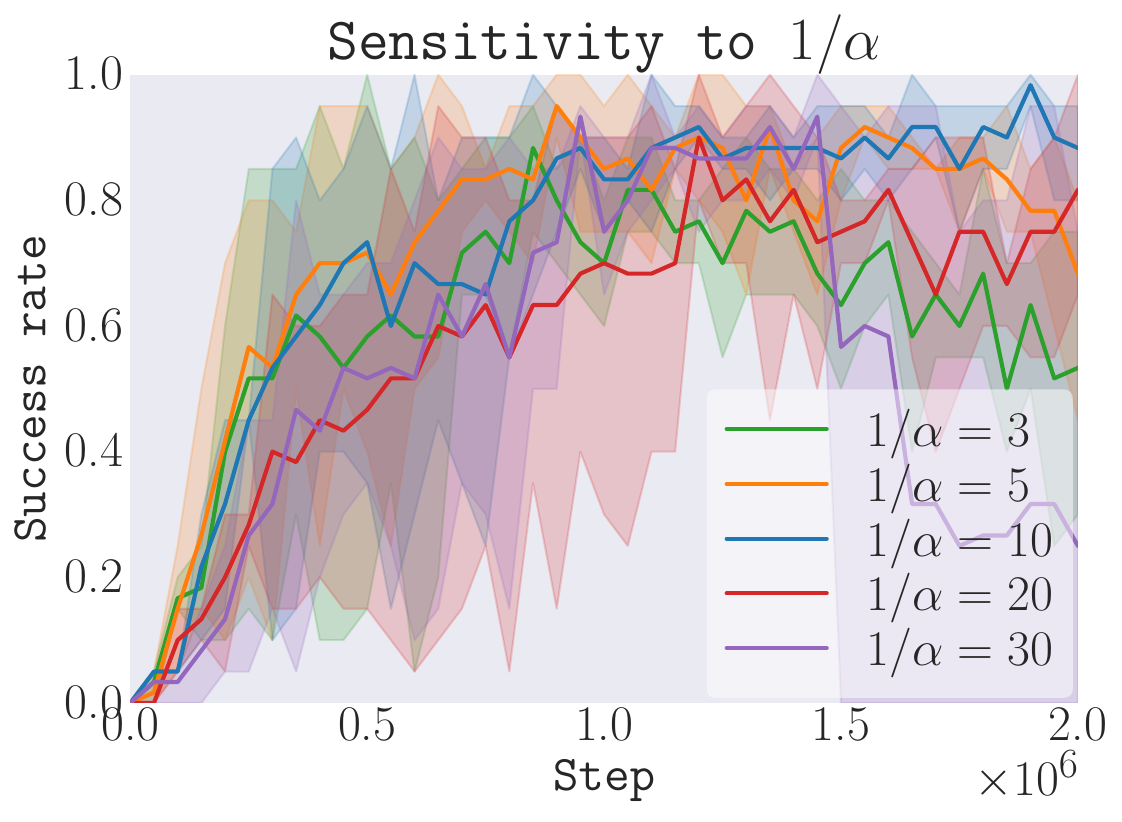}
    \includegraphics[width=.32\linewidth]{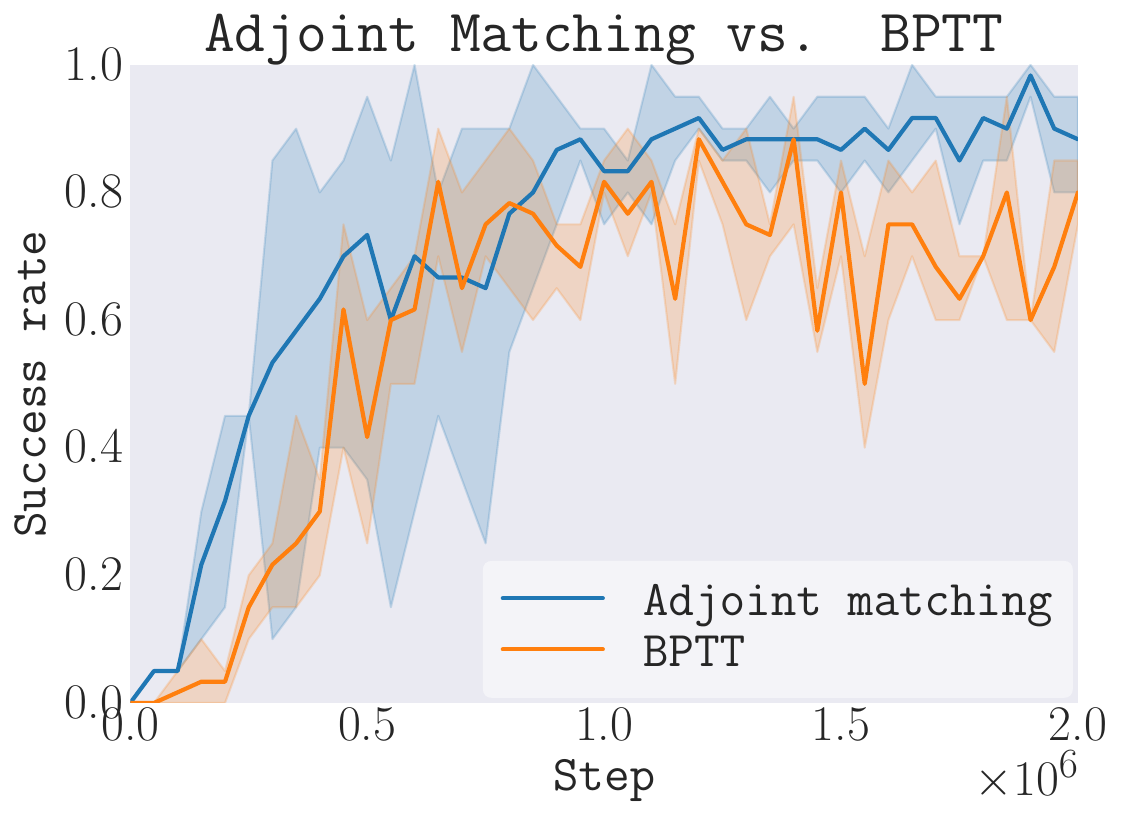}
    \includegraphics[width=.32\linewidth]{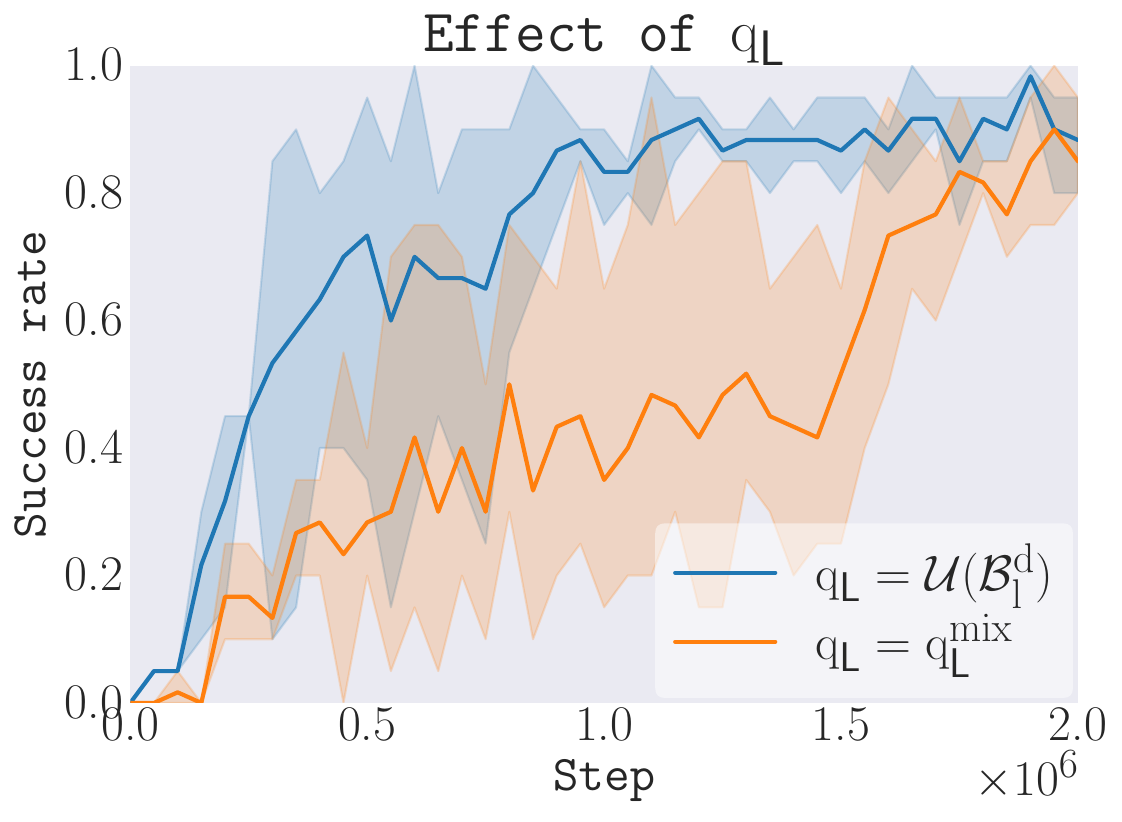}
    \caption{\textbf{Ablation studies in} \envname{cube-double-play-singletask-task2-v0}\textbf{.}
    Moderate values of the inverse temperature $1/\alpha\in[5, 20]$ yield comparable results.
    Adjoint matching achieves better and more stable results than BPTT.
    Using a different $\LatentBCDist{}$
    yields comparable final success rates in this task, consistent with the proposed separation between entropy regularization and behavior regularization.
    Lines show mean success rates, and shading spans the minimum and maximum across runs.
    }
    \label{fig: ablations}
\end{figure}

\textbf{(Q2): The key hyperparameter of \algo{} is the inverse temperature.}
Using a constant set of method-specific hyperparameters for \algo{} across all tasks demonstrates its robustness to different tasks and datasets. Nevertheless, we analyze the sensitivity of \algo{} to the inverse temperature $1/\alpha$ (\Cref{fig: ablations}).
Values of \mbox{$1/\alpha\in[5, 20]$} yield comparable results.
At $1/\alpha=3$, performance drops on both datasets, whereas $1/\alpha=30$ degrades performance on \datasetname{clean} data but remains competitive on \datasetname{noisy} data (\Cref{tab: inverse-temperature-tps}).
We find that \mbox{$1/\alpha=10$} yields stable results for all environments.
We also compare \ourname{} with $Q$-only optimization and TPS in this ablation and find that \ourname{} with \mbox{$1/\alpha\in[5, 20]$} outperforms both of them (\Cref{tab: inverse-temperature-tps}).
In addition, we analyze sensitivity to $\sigma_0$ and find that \ourname{} is robust to this hyperparameter (\Cref{app: experiment-additional-results}).

\textbf{(Q3): Adjoint matching achieves better results and is faster than BPTT using reparametrized gradients.}
A key advantage of adjoint matching is that it circumvents the need to perform BPTT on \eqref{eq: flow-entropy}, which is computationally expensive and numerically unstable.
We validate this empirically and observe that adjoint matching yields more stable training and stronger performance compared to BPTT (\Cref{fig: ablations}) while training $3.6\times$ faster ($123\,\mathrm{it/s}$ vs $34\,\mathrm{it/s}$).

\textbf{(Q4): The latent space successfully decouples entropy regularization from behavior regularization.}
To understand the improvements of \ourname{}, we change $\LatentBCDist{}$ to a distribution with the same support but a different density, which changes the behavior regularization effect but keeps the same latent-space entropy regularization (see \Cref{app: latent-mixture}).
{Both runs retain the uniform entropy reference on the same latent ball. Changing the decoder prior can still alter the learned latent-to-action mapping and learning dynamics, while leaving the entropy regularizer unchanged.}
We observe that changing the decoder prior yields comparable final success rates in this task, consistent with the proposed separation between latent-space entropy regularization and behavior regularization (\Cref{fig: ablations}).

\section{Conclusion}\label{sec: conclusion}

We introduced \algo{}, an entropy-regularized approach to support-constrained flow policy optimization for offline RL.
Our key insight is that support constraints and entropy regularization solve different problems: the former prevent OOD actions, while the latter stabilizes policy extraction within the learned dataset support.
By combining a flow-based decoder with latent-space adjoint matching, \algo{} effectively helps prevent mode collapse, smooths the optimization landscape, and avoids local minima, demonstrating state-of-the-art performance on the OGBench suite with fixed method-specific hyperparameters.

\textbf{Limitations and Future Work.}
We identify several promising avenues for extension.
First, although adjoint matching circumvents the need for BPTT, solving the ODEs needed to sample from the flow model during training remains computationally intensive for deep neural ODEs. This computational cost could be reduced through distillation \citep{sorl}.
Second, by relying on the decoder to enforce support constraints,
\algo{} inherits OOD errors made by the decoder itself.
Integrating behavior cloning methods with stricter support constraints, diagnosing when the BC model generates OOD errors, or applying a pre-trained robust BC model could mitigate this.
In addition, \algo{} uses a simple value-function learning method and an actor-critic architecture similar to those in \cite{fql_park2025}. Both could be improved using ideas from \citep{mao2023svr,xql,liu2024adaptive,agrawalla2025floq}.

\textbf{Broader Impacts.}\label{sec: broader-impacts}
Learning policies from existing datasets may reduce the cost and risks of online data collection in robotics. However, support constraints alone do not guarantee safe behavior, as datasets may contain unsafe actions and learned decoders can make errors. Deployment under distribution shift could therefore cause physical harm. Safety validation and monitoring remain necessary before deploying these policies on physical systems.

\section*{Acknowledgements}

This material is based upon work supported by the Under Secretary of War for Research and Engineering under Air Force Contract No. FA8702-15-D-0001 or FA8702-25-D-B002. In addition,
Zhang, So, and Fan are supported by the National Science Foundation (NSF) CAREER Award
\#CCF-2238030. This work was also supported in part by a grant from Amazon. Any opinions, findings, conclusions or recommendations expressed in this material are those of the author(s) and do not necessarily reflect the views of the Under Secretary of War for Research and Engineering.

\copyright{} 2026 Massachusetts Institute of Technology.

Delivered to the U.S. Government with Unlimited Rights, as defined in DFARS Part 252.227-7013 or 7014 (Feb 2014). Notwithstanding any copyright notice, U.S. Government rights in this work are defined by DFARS 252.227-7013 or DFARS 252.227-7014 as detailed above. Use of this work other than as specifically authorized by the U.S. Government may violate any copyrights that exist in this work.

\newpage
\clearpage

\appendix
\crefalias{section}{appendix}
\crefalias{subsection}{appendix}

\section{Extended Related Work}\label{app: extended_related}
\textbf{Mitigating distribution shift in offline RL.}
A central challenge in offline RL is mitigating distribution shift: learned value functions can overestimate actions that are not in the dataset, causing the policy to select these out-of-distribution actions.
Classical approaches address this issue by either regularizing the policy to remain close to the behavior policy with statistical distances (e.g., KL divergence \citep{wang2018exponentially,peng2019advantage,wu2019behavior,jaques2019way,siegel2020keep,nair2020awac,wang2020critic,kostrikov2022offline}, MMD distance \citep{kumar2019stabilizing} or Wasserstein distance \citep{wu2019behavior}),
or regularizing the value function to be pessimistic on out-of-distribution actions \citep{kumar2020conservative,kostrikov2022offline,lyu2022mildly,mao2023svr,cen2024learning,wang2026beyond}.
Another family of methods addresses the OOD problem by penalizing actions that fall outside the support of the behavior policy \citep{kumar2019stabilizing,wu2022spot,mao2023str,zhang2023constrained}.
However, these methods rely on heuristic sample-based proxies \citep{kumar2019stabilizing} or Gaussian density approximations \citep{wu2022spot,mao2023str,zhang2023constrained} which can fail to capture expressive multimodal behavior policies.

\textbf{Expressive generative policies for offline RL.}
Diffusion \citep{song2020score,ho2020denoising} and flow-matching \citep{lipman2023flow,albergo2022building,liu2023flow} models have recently
been used as expressive alternatives to deterministic or Gaussian policies,
especially in robotic imitation learning where action distributions are often
multimodal \citep{chi2023diffusion,reuss2023goal,pearce2023imitating,wang2023diffusion}.
These policies have also been adapted to offline RL, where the goal is to
improve beyond the behavior policy using a learned critic.
However, because flow and diffusion policies generate actions through an iterative sampling process, critic-based optimization is more challenging than for standard Gaussian actors.
Existing reward-optimization methods for generative policies can be broadly
divided by how they use the learned critic: \emph{critic-evaluation methods},
which use only scalar values $Q(s,a)$ to select or reweight actions, and
\emph{critic-gradient methods}, which use $\nabla_a Q(s,a)$ to directly
improve generated actions.
\begin{enumerate}[label=(\roman*), leftmargin=*, labelindent=0pt, widest=ii, topsep=0pt, parsep=0pt]
    \item \textit{Critic-evaluation} methods include sampling-based methods that imitate the behavior policy and use the critic to sample from the reward-tilted distribution either at test time \citep{chen2023offline,hansen2023idql} or during training, when the sampled actions serve as imitation targets \citep{chen2024score}.
    These methods often reliably improve the quality of the extracted policy at greater computational cost and depend on the quality of the behavior policy.
    Weighted behavioral cloning methods also fall under this category.
    These methods use the critic to assign per-sample weights in the diffusion or flow-matching loss, selectively cloning high-value actions in the dataset \citep{kang2023efficient,ding2024diffusion,Zhang2025EnergyWeightedFlowMatching},
    but generally do not optimize the learned policy towards the reward-tilted distribution \citep{li2026qam}.
    \item Some \textit{critic-gradient} methods combine a behavior-cloning term with a term that increases the critic-estimated value.
    These methods require backpropagation through the iterative sampling process \citep{wang2023diffusion,ding2024consistency,zhang2024entropy,ada2024diffusion}, which can be computationally expensive and unstable \citep{fql_park2025,sorl,li2026qam}. Other methods learn a distilled policy \citep{fql_park2025,tiofack2026gfp,chae2026fac,cheng2026fisher,zhang2026goldenstart} at the expense of policy expressiveness.
    Another line of work uses guidance \citep{dhariwal2021diffusion,ho2022classifier} during the iterative sampling process to favor high-value actions \citep{lu2023contrastive,frans2025diffusion,xu2026vgf}, but these methods rely on approximations that do not have theoretical guarantees \citep{du2023reduce,li2026qam}.
    Most closely related to our optimization mechanism, \citep{li2026qam} provides an efficient way to fine-tune flow-based policies using critic gradients and a KL-regularized reward objective via adjoint matching \citep{domingo2024adjoint}.
    However, unlike critic-evaluation methods where samples from the learned policy are either directly cloned (sampling-based methods) or selectively cloned (weighted BC methods), these methods rely on statistical distance regularization to avoid the OOD problem (e.g., L2 / Wasserstein distance \citep{fql_park2025} or KL regularization \citep{li2026qam}).
    A key challenge is that the regularization weight is a hyperparameter that significantly affects the performance and needs to be tuned for each task and dataset \citep{ogbench_park2025,fql_park2025}.
\end{enumerate}
Our method falls under the \textit{critic-gradient} category as we also use adjoint matching to optimize the policy with critic gradients under KL regularization.
However, to address the OOD problem, \algo{} relies on the latent space to enforce the support constraints by construction instead of using KL regularization toward the behavior policy, as in \citep{li2026qam}, which can be too restrictive in cases where optimal actions have very low probability under the behavior policy \citep{kumar2019stabilizing}.

\textbf{Support constraints in offline RL.}
Instead of addressing the OOD problem with regularization, another line of work proposes to address the OOD problem by learning a bounded latent space that maps to the support of the behavior policy, then performing policy extraction in this latent space to restrict the support of the learned policy by construction \citep{zhou2021plas,wagenmaker2025steering,zhang2026reform}.
Our work follows this support-constrained perspective, but revisits the optimization problem that remains after support has been controlled.
Namely, we find that entropy regularization is essential to prevent policy collapse, and leverage adjoint matching \citep{domingo2024adjoint} to perform entropy-regularized policy extraction in latent space with an \textit{expressive} flow policy that achieves a better performance-entropy tradeoff than Gaussian policies.

\textbf{Entropy-regularized policy extraction.}
The role of entropy regularization has been well identified in \textit{online} RL as a way to improve exploration \citep{ziebart2010modeling,haarnoja2018soft} and policy robustness \citep{ziebart2010modeling,eysenbach2022maximum} and to smooth the optimization landscape \citep{ahmed2019understanding}.
In \textit{offline} RL,
while some works note the benefits of entropy regularization for exploration \citep{ding2024diffusion,dong2025expo},
the importance of entropy regularization has been less emphasized, perhaps due to the implicit entropy regularization effect from various behavior policy regularization techniques.
In our setting, the learned decoder and latent support constraint handle support control, so no separate behavior-policy regularization term is needed. This makes the role of entropy regularization in policy extraction explicit.

\section{Proof of \Cref{thm: ng-target-optimal}} \label{app: ng-target-optimal:proof}

For a Lebesgue-measurable set $X \subseteq \mathbb{R}^d$, let $\lambda(X)$ denote its Lebesgue measure.
We first prove the following technical lemma to express entropy as a constant minus a KL divergence.
\begin{lemma} \label{lem: entropy-as-kl}
    Fix a state $s \in \mathcal{S}$. For brevity, we omit the state dependence in the policy.
    Let $\LatentSpace_q \coloneqq \supp(\LatentBCDist{})\subseteq \LatentSpace{}$ denote the support of $\LatentBCDist{}$.
    Suppose $\LatentSpace_q$ has finite positive Lebesgue measure, i.e., $0 < \lambda(\LatentSpace_q) < \infty$,
    and $\LatentPolicy{}$ has a density and satisfies the latent-space support constraints \eqref{eq: latent-support-constraint}, i.e.,
    \begin{equation}
        \supp(\LatentPolicy{}) \subseteq \supp(\LatentBCDist{}).
    \end{equation}
    Let $\LatentUniform{}$ be the uniform distribution over $\supp(\LatentBCDist{})$, i.e.,
    \begin{equation}
        \LatentUniform{}(z) = \begin{cases}
            1 / \lambda(\LatentSpace_q) & z \in \LatentSpace_q \\
            0 & \text{otherwise}
        \end{cases}.
    \end{equation}
    Then, the entropy of $\LatentPolicy{}$ can be expressed as
    \begin{equation} \label{eq: entropy-as-kl}
        \mathcal{H}(\LatentPolicy{}) = -\KL( \LatentPolicy{} \| \LatentUniform{} ) + \log \lambda(\LatentSpace_q),
    \end{equation}
    where the second term in \eqref{eq: entropy-as-kl} is a constant independent of $\LatentPolicy{}$.
\end{lemma}
\begin{proof}
    Expanding the KL divergence,
    \begin{align}
        \KL( \LatentPolicy{} \| \LatentUniform{} )
        &= \int_{\mathbb{R}^d} \LatentPolicy{}(z) \log \frac{\LatentPolicy{}(z)}{ \LatentUniform{}(z) } \diff z, \\
        &= \int_{\LatentSpace_q} \LatentPolicy{}(z) \log \LatentPolicy{}(z) \diff z - \int_{\LatentSpace_q} \LatentUniform{}(z) \log \LatentUniform{}(z) \diff z, \\
        &= -\mathcal{H}(\LatentPolicy{}) - \int_{\LatentSpace_q} \LatentUniform{}(z) \log\left( \frac{1}{\lambda(\LatentSpace_q)} \right) \diff z, \\
        &= -\mathcal{H}(\LatentPolicy{}) + \log \lambda(\LatentSpace_q).
    \end{align}
    Rearranging the above equation yields the desired result.
\end{proof}

We are now ready to prove \Cref{thm: ng-target-optimal}.
\begin{proof}
    Fix a state $s \in \mathcal{S}$. For brevity, we omit the state dependence in the policy.
    We write $Q^{\LatentPolicy{}}(z)$ for $Q_\phi^{\LatentPolicy{}}(s,z)$, held fixed with respect to the candidate policy.
    Restating the entropy-regularized optimization problem \eqref{eq: entropy-reg-pol-loss} in terms of the policy $\LatentPolicy{}$ directly,
    we wish to solve the following optimization problem over densities supported on $\LatentSpace_q$
    \begin{equation} \label{eq: entropy-reg-pol-loss-latent-repeat}
        \begin{gathered}
        \min_{\LatentPolicy{}} -\mathbb E_{z\sim \LatentPolicy{}}\left[Q^{\LatentPolicy{}}(z)\right] - \alpha \mathcal H(\LatentPolicy{}).
        \end{gathered}
    \end{equation}
    By Lemma~\ref{lem: entropy-as-kl}, the entropy of $\LatentPolicy{}$ can be expressed as
    \begin{equation}
        \mathcal{H}(\LatentPolicy{}) = -\KL( \LatentPolicy{} \| \LatentUniform{} ) + \log \lambda(\LatentSpace_q).
    \end{equation}
    Substituting this into \eqref{eq: entropy-reg-pol-loss-latent-repeat} and dropping the constant term, we now have a KL-regularized optimization problem
    \begin{equation}
        \begin{gathered}
        \min_{\LatentPolicy{}} -\mathbb E_{z\sim \LatentPolicy{}}\left[Q^{\LatentPolicy{}}(z)\right] + \alpha \KL(\LatentPolicy{} \| \LatentUniform{}).
        \end{gathered}
    \end{equation}
    Now, define the partition function $Z$ as
    \begin{equation}
        Z \coloneqq \int_{\mathbb{R}^d} \LatentUniform{}(z) \exp\left( \frac{1}{\alpha} Q^{\LatentPolicy{}}(z) \right) \diff z.
    \end{equation}
    Since $Q^{\LatentPolicy{}}$ is measurable and bounded from above, there exists an $m \in \mathbb{R}$ such that $Q^{\LatentPolicy{}}(z) \leq m$ for all $z \in \LatentSpace_q$. Hence,
    \begin{equation}
        0 < \exp\left( \frac{1}{\alpha} Q^{\LatentPolicy{}}(z) \right) \leq \exp\left( \frac{1}{\alpha} m \right) < \infty,
    \end{equation}
    and thus
    \begin{equation}
        0 < Z \leq \exp\left( \frac{1}{\alpha} m \right) < \infty,
    \end{equation}
    so $\LatentPolicy^\star$, defined as
    \begin{equation}
        \LatentPolicy^\star(z) \coloneqq \frac{1}{Z} \LatentUniform{}(z) \exp\left( \frac{1}{\alpha} Q^{\LatentPolicy{}}(z) \right),
    \end{equation}
    is well-defined.
    Since $\KL( \LatentPolicy{} \| \LatentUniform{} ) = +\infty$ for any $\LatentPolicy{} \centernot{\ll} \LatentUniform{}$, it is sufficient to consider policies $\LatentPolicy{}$ that are absolutely continuous with respect to $\LatentUniform{}$. Note that for $z \in \LatentSpace_q$,
    \begin{equation}
        \frac{ \LatentPolicy{}^\star(z) }{ \LatentUniform{}(z) } = \frac{1}{Z} \exp\left( \frac{1}{\alpha} Q^{\LatentPolicy{}}(z) \right), \qquad \log\left( \frac{ \LatentPolicy{}^\star(z) }{ \LatentUniform{}(z) } \right) = -\log Z + \frac{1}{\alpha} Q^{\LatentPolicy{}}(z).
    \end{equation}
    For any $\LatentPolicy{} \ll \LatentUniform{}$, we have
    \begin{align}
        \KL( \LatentPolicy{} \| \LatentPolicy{}^\star )
        &= \int_{\mathbb{R}^d} \LatentPolicy{}(z) \log \frac{\LatentPolicy{}(z)}{ \LatentPolicy{}^\star(z) } \diff z, \\
        &= \int_{\mathbb{R}^d} \LatentPolicy{}(z) \log \frac{\LatentPolicy{}(z)}{ \LatentUniform{}(z) } \diff z - \int_{\mathbb{R}^d} \LatentPolicy{}(z) \log \frac{ \LatentPolicy{}^\star(z) }{ \LatentUniform{}(z) } \diff z, \\
        &= \KL( \LatentPolicy{} \| \LatentUniform{} ) + \log Z - \frac{1}{\alpha} \int_{\mathbb{R}^d} \LatentPolicy{}(z) Q^{\LatentPolicy{}}(z) \diff z, \\
        &= \KL( \LatentPolicy{} \| \LatentUniform{} ) + \log Z - \frac{1}{\alpha} \mathbb E_{z\sim \LatentPolicy{}}[Q^{\LatentPolicy{}}(z)].
    \end{align}
    Thus, the objective can be written as the KL divergence to $\LatentPolicy{}^\star$ plus a constant:
    \begin{equation}
        -\E_{z \sim \LatentPolicy{}}[Q^{\LatentPolicy{}}(z) - \alpha \KL( \LatentPolicy{} \| \LatentUniform{} )] = \alpha \KL( \LatentPolicy{} \| \LatentPolicy{}^\star ) - \alpha \log Z.
    \end{equation}
    Since $\KL( \LatentPolicy{} \| \LatentPolicy{}^\star ) \geq 0$ for all $\LatentPolicy{}$ with equality if and only if $\LatentPolicy{} = \LatentPolicy{}^\star$, we conclude that $\LatentPolicy{}^\star$ is the unique optimal solution to the optimization problem \eqref{eq: entropy-reg-pol-loss-latent-repeat}.
\end{proof}

\section{More Theoretical Results}

\subsection{The Static Flow Issue and Implicit Entropy Regularization in ReFORM}\label{app: reform-issue}

To enforce the latent-space support constraint \eqref{eq: latent-support-constraint}, ReFORM \citep{zhang2026reform} introduces a reflection term during the integration of the latent flow policy $\muNToL$. The corresponding modified Euler step is defined as:
\begin{equation}\label{eq: reform-modified-euler-step}
    w_{k+1} = \mathbbm{1}\{\hat w_{k+1}\in\mathcal B_l^d\}\hat w_{k + 1} + (1 - \mathbbm{1}\{\hat w_{k+1}\in\mathcal B_l^d\})\left(\hat w_{k+1} - \langle v_\thetaNToL(w_k, k\Delta t; s)\Delta t, n_{k+1}\rangle n_{k+1}\right),
\end{equation}
where $\hat w_{k+1} = w_k + v_\thetaNToL(w_k, k\Delta t; s)\Delta t$ denotes the standard Euler step, $N$ is the total number of integration steps, $k\in\{0, \dots, N - 1\}$, $\Delta t = \frac{1}{N}$, $n_{k+1} = \frac{\hat w_{k+1}}{\|\hat w_{k+1}\|}$, and $\langle\cdot,\cdot\rangle$ represents the inner product. Integrating this modified step yields the final latent variable $z \gets w_{N}$.

While \cite{zhang2026reform} proves that this modified integration step guarantees adherence to the latent support constraint \eqref{eq: latent-support-constraint}, it introduces a degenerate behavior we term ``static flow''. This occurs when the proposed step violates the boundary ($\hat w_{k + 1}\notin\mathcal B_l^d$) and the flow velocity is strictly codirectional with the outward normal vector ($v_\thetaNToL \upuparrows n_{k+1}$). Under these conditions, the reflection term perfectly cancels the velocity step:
\begin{equation}
    \begin{aligned}
        w_{k+1} & = \hat w_{k+1} - \langle v_\thetaNToL(w_k, k\Delta t; s)\Delta t, n_{k+1}\rangle n_{k+1} \\
        & = \left(w_k + v_\thetaNToL(w_k, k\Delta t; s)\Delta t\right) - v_\thetaNToL(w_k, k\Delta t; s)\Delta t = w_k.
    \end{aligned}
\end{equation}
Consequently, the corrected step is stationary under this alignment condition. Rather than following the flow field $v_\thetaNToL$ to the boundary of $\mathcal B_l^d$, the pushforward distribution of $\muNToL$ unnaturally accumulates along the radial trajectory. To illustrate this phenomenon, we train ReFORM on the 2D toy environment introduced in \Cref{fig: ent-q-tradeoff}. The learned latent distribution and its resulting action distribution are visualized in \Cref{fig: reform-analysis}. We empirically observe that the latent mass inappropriately concentrates along two radii, resulting in a suboptimal action distribution.
We also conduct additional experiments in \envname{cube-double} in \Cref{app: experiment-additional-results} to measure how often the boundary correction is activated.

Interestingly, we find that this static flow artifact unintentionally serves as an implicit entropy regularizer, which partially explains ReFORM's empirical success. As demonstrated in \Cref{fig: reform-analysis}, the premature halting of trajectories prevents the latent distribution from collapsing to a single point, artificially maintaining the spread of the distribution.

\begin{figure}[t]
    \centering
    \includegraphics[width=\linewidth]{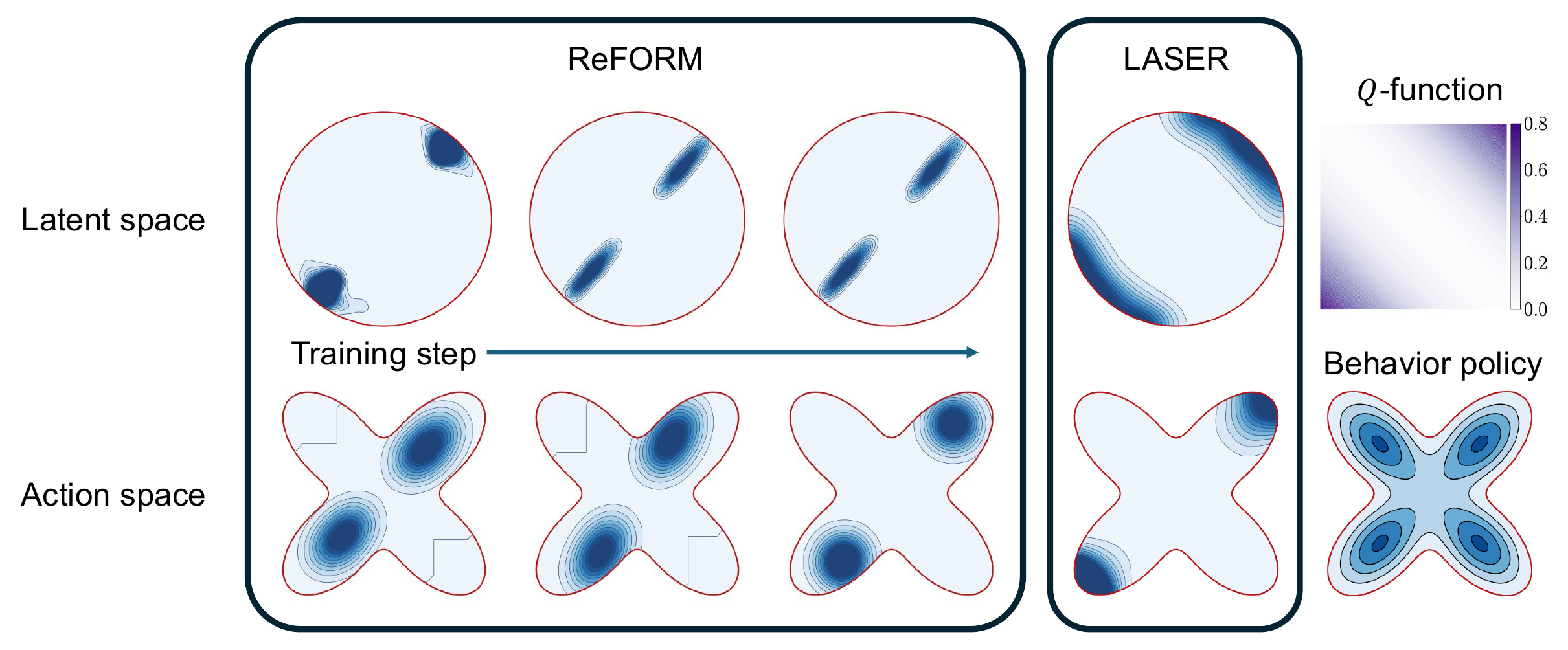}
    \caption{\textbf{ReFORM fails to maximize the $Q$-value because of the ``static flow'' issue.} For a given behavior policy and a ground-truth $Q$-function, the generated latent-space distribution $(\muNToLCondS)_\sharp\BaseDist{}$ concentrates along radial directions aligned with the flow, rather than moving toward the boundary of $\mathcal B_l^d$. This makes the flow static, resulting in suboptimal performance compared with \algo{}.}
    \label{fig: reform-analysis}
\end{figure}

\subsection{The Ground-truth Base Velocity Field from the Base Space to the Latent Space}\label{app: gt-flow-gaussian-uniform}

As mentioned in \Cref{sec: full-algo}, there exists a ground-truth velocity field $v_{\NToL}$ that maps samples from $w\sim\BaseDist{} = \mathcal N(0, \sigma_0^{2} I)$ to $z\sim\mathcal U(\mathcal B_l^d)$. We show here how we can find such a ground-truth velocity field and why we choose not to use it in our implementation.

Since both the base distribution $\BaseDist{} = \mathcal N(0, \sigma_0^2 I)$ and the target distribution $\mathcal U(\mathcal B_l^d)$ are spherically symmetric, the optimal transport (OT) mapping between them operates purely along the radial direction. Let $r_w = \|w\|$ denote the radius of a sample $w$ from the base distribution, and $r_z = \|z\|$ denote the radius in the latent space. The cumulative distribution function (CDF) for the uniform distribution on a $d$-dimensional ball of radius $l$ is $F_\Uniform(r) = (r/l)^d$ for $r \in [0, l]$. Meanwhile, the radius of the Gaussian base distribution follows a scaled $\chi$-distribution with $d$ degrees of freedom, with its CDF denoted as $F_\chi(r)$.

By matching the quantiles of these radial distributions, we can construct the exact analytical mapping $\mu_\NToL: \mathbb{R}^d \to \mathcal B_l^d$ as follows:
\begin{equation}
    \mu_\NToL(w) = \frac{w}{\|w\|} F_\Uniform^{-1}(F_\chi(\|w\|)) = w \frac{l \cdot F_\chi(\|w\|)^{1/d}}{\|w\|}.
\end{equation}
At the origin, we define $\mu_\NToL(0)=0$ by continuity. However, although we have derived an exact analytical mapping $\mu_\NToL$, it is nontrivial to derive the corresponding velocity field $v_{\NToL}(\pl {w_t}, \pl t; \pl s)$, i.e., the instantaneous velocity of a particle currently at location $w_t$ at time $t$.
Under standard optimal transport displacement interpolation, a particle travels in a straight line from its origin $w$ to its destination $\mu_\NToL(w)$ at a constant speed, yielding the trajectory:
\begin{equation}
    w_t = (1 - t)w + t\mu_\NToL(w).
\end{equation}
While the velocity of this specific particle is trivially $v(w) = \mu_\NToL(w) - w$, the velocity field must be expressed independently of the origin $w$. This requires inverting the trajectory equation to express $w$ purely as a function of the current state $w_t$ and time $t$.

Because our mapping $\mu_\NToL$ is purely radial, the vectors $w$ and $w_t$ share the same direction, meaning $\frac{w}{\|w\|} = \frac{w_t}{\|w_t\|}$ for $w\ne0$. Let $r_w = \|w\|$ and $r_t = \|w_t\|$. Taking the norm of the trajectory equation yields the following relationship between the magnitudes:
\begin{equation}
    r_t = (1 - t)r_w + t \cdot l \cdot F_\chi(r_w)^{1/d}.
\end{equation}
To compute the velocity at $w_t$, we must find the inverse function $g_t(r_t)$ that recovers the initial radius $r_w$ given $r_t$ and $t$. For each $t\in[0,1]$, the radial map is strictly increasing, so it has a unique inverse on its image. For $t>0$ and $w_t\ne0$, the exact Eulerian velocity field is:
\begin{equation}
    v^\mathrm{base}_{\NToL}(w_t, t) = \frac{w_t - w}{t} = \left(\frac{\|w_t\| - g_t(\|w_t\|)}{t}\right) \frac{w_t}{\|w_t\|}.
\end{equation}
The values at $t=0$ and $w_t=0$ are defined by continuity.
Thus, although $g_t$ exists, evaluating it generally requires numerical inversion.
For intermediate times $0<t<1$, computing $g_t(\|w_t\|)$ would add a numerical solve at each velocity evaluation. We instead learn $\BaseVelocity^\mathrm{base}$ via flow matching to avoid these repeated solves.

\section{Practical Implementation of \algo{}}\label{app: implementation}

We provide a high-level summary of \algo{} in \Cref{alg: laser}.

\begin{algorithm}
\caption{\algo{}} \label{alg: laser}
\begin{algorithmic}[1]
\While{not converged}
  \LComment{Train Critic}
  \State Sample batch $\{ (s, a, r, s') \} \sim \mathcal{D}$
  \State $w \sim \mathcal{N}(0, \sigma_0^2 I_d)$
  \State $a' = \mu_{\thetaLToA}^{s'}(\mu_{\thetaNToL}^{s'}(w))$
  \State Update $\textcolor{TrainingColor}{\phi}$ to minimize $L(\textcolor{TrainingColor}{\phi}) = \mathbb{E}[ \big( r(s,a) + \gamma Q^{\FullPolicy}_{\hat{\phi}}(s', a') - Q^{\FullPolicy}_{\textcolor{TrainingColor}{\phi}}(s, a) \big)^2 ]$ \quad \eqref{eq: critic-loss}
  \Statex
  \LComment{Train vector field $v_{\textcolor{TrainingColor}{\thetaLToA}}$ in decoder $\mu_{\textcolor{TrainingColor}{\thetaLToA}}$}
  \State $x^0 \sim \LatentBCDist{}$
  \State $x^1 \gets a$
  \State $t \sim \Uniform([0, 1])$
  \State $x^t \gets (1-t) x^0 + t x^1$
  \State Update $\textcolor{TrainingColor}{\thetaLToA}$ to minimize $L_\LToA(\textcolor{TrainingColor}{\thetaLToA}) = \mathbb{E}\left[ \norm{v_{\textcolor{TrainingColor}{\thetaLToA}}(x^t, t; s) - (x^1 - x^0)}^2 \right]$ \quad \eqref{eq: flow-matching}
  \Statex
  \LComment{Train vector field $\textcolor{TrainingColor}{\vNToLBase}$ to match $\LatentUniform{}$}
  \State $x^0 \sim \BaseDist{}{}$
  \State $x^1 \sim \LatentUniform{}$
  \State $t \sim \Uniform([0, 1])$
  \State $x^t \gets (1-t) x^0 + t x^1$
  \State Update $\textcolor{TrainingColor}{\vNToLBase}$ to minimize $\mathbb{E}\left[ \norm{\textcolor{TrainingColor}{\vNToLBase}(x^t, t) - (x^1 - x^0)}^2 \right]$
  \Statex
  \LComment{Train Latent-Space Policy using Adjoint Matching}
  \State Set $\Delta t=1/N$ and $t_i=i\Delta t$ for $i=0,\dots,N$
  \State $w_0 \sim \mathcal{N}(0, \sigma_0^2 I)$
  \State Forward simulate using \(\eqref{eq: forward-sde-EM}\) to obtain \(\{w_{t_i}\}_{i=1}^N\), held fixed during policy regression.
  \State $\tilde{g}_1 \gets -\frac{1}{\alpha} \nabla_{w_1} Q_{\hat\phi}^{\LatentPolicy{}}(s, w_1)$
  \State Compute $\tilde g_{t_i}$ for $i=N-1,\dots,0$ using \eqref{eq: lean-adjoint-EM}
  \State Compute detached regression targets $\hat v^\star_{t_i}$ for $i=0,\dots,N-1$ using \eqref{eq: am-target-discrete}
  \State Update $v_{\textcolor{TrainingColor}{\thetaNToL}}$ to minimize\\
    \makebox[\linewidth][c]{${L}_\mathrm{AM}(\textcolor{TrainingColor}{\thetaNToL}) = \mathbb{E} \left[ \frac{1}{2N}\sum_{i=0}^{N-1} \norm[\Big]{ \frac{2}{\bar\sigma_{t_i}} \Big( v_{\textcolor{TrainingColor}{\thetaNToL}}(w_{t_i}, t_i; s) - \hat{v}^\star_{t_i} \Big) }^2 \right]$ \quad\eqref{eq: am}}
\EndWhile
\State \Return policy $\FullPolicy{}$
\end{algorithmic}
\end{algorithm}

\paragraph{Flow-matching Decoder Learning.}
We parameterize the velocity field $v_\thetaLToA$ of the decoder $\muLToA$ with a multilayer perceptron (MLP), and learn it via the flow-matching loss \eqref{eq: flow-matching}. The prior latent-space distribution is chosen to be a bounded uniform distribution over a $d$-ball with radius $l$, i.e., $\LatentBCDist{} = \mathcal U(\mathcal B_l^d)$. After training, we sample $z\sim\LatentBCDist{}$ and use Euler integration to get the action samples:
\begin{equation}
    z_0 \gets z, \quad z_{(k+1)\Delta t} = z_{k\Delta t} + v_\thetaLToA(z_{k\Delta t}, k\Delta t; s)\Delta t,
\end{equation}
where $N$ is the total number of integration steps, $\Delta t = \frac{1}{N}$, $k \in \{0, \dots, N-1\}$, and the final generated action is $a\gets z_1$.

\paragraph{Flow matching for $\LatentUniform{}$.}
Similarly, the velocity field $\vNToLBase{}$ is parameterized by an MLP and learned via the flow-matching loss \eqref{eq: flow-matching}, with prior distribution $\BaseDist{} = \mathcal N(0, \sigma_0^2I)$ and the desired pushforward $\LatentUniform{} = \mathcal U(\mathcal B_l^d)$. Note that $\vNToLBase{}$ is state-independent.

\paragraph{Critic Learning.}
The critic is learned via the standard TD learning objective \eqref{eq: critic-loss} with an ensemble of $K$ critics. In some environments, we found that using a pessimistic target-value backup \citep{ghasemipour2022so} can improve performance. In that case, the loss function becomes
\begin{equation}
    L(\phi_j) = \mathbb E_{(s, a, s')\sim\mathcal D, a'\sim\pi_\textsf{A}(\cdot|s')}\left[\left(Q_{\phi_j}(s, a) - r - \gamma\left(Q_{\hat\phi}^\mathrm{mean}(s', a') - \rho Q_{\hat\phi}^\mathrm{std}(s', a')\right)\right)^2\right],
\end{equation}
where $\phi_j$ denotes the parameters of the $j$th critic, $Q_{\hat\phi}^\mathrm{mean}(s', a')\coloneqq\frac{1}{K}\sum_k Q_{\hat \phi_{k}}(s', a')$, $Q_{\hat\phi}^\mathrm{std}(s', a') = \sqrt{\frac{1}{K}\sum_k\left(Q_{\hat\phi_{k}}(s', a') - Q_{\hat\phi}^\mathrm{mean}(s', a')\right)^2}$, and $\rho$ is the pessimistic backup weight.

\paragraph{Entropy-regularized Policy Optimization via Adjoint Matching.}
We adapt the adjoint-matching implementation \citep{domingo2024adjoint}, retaining its endpoint regularization and final noiseless reference-flow step. We discretize the forward SDE and backward adjoint ODE on $t_i=i\Delta t$, with $\Delta t=1/N$ and $N=10$. To avoid singular endpoint coefficients, we replace the denominator $t_i$ by $t_i+\Delta t$ and use the regularized noise schedule $\bar\sigma_{t_i}=\sigma_0\sqrt{2(1-t_i+\Delta t)/(t_i+\Delta t)}$. Starting from $w_0\sim\mathcal N(0,\sigma_0^2I)$, the first $N-1$ steps use Euler integration. The final step uses the reference flow and then projects its output onto the latent ball, ensuring $w_1\in\mathcal B_l^d$:
\begin{equation} \label{eq: forward-sde-EM}
    \begin{aligned}
        w_{t_{i+1}} &\gets w_{t_i} + \left(2v_\thetaNToL(w_{t_i}, t_i; s) - \frac{w_{t_i}}{t_i+\Delta t}\right)\Delta t + \sqrt{\Delta t}\,\bar\sigma_{t_i}n_i,\quad i=0,\dots,N-2,\\
        w_1 &\gets \Pi_{\mathcal B_l^d}\!\left(w_{t_{N-1}} + \Delta t\,\vNToLBase(w_{t_{N-1}},t_{N-1})\right),
    \end{aligned}
\end{equation}
where $n_i\sim\mathcal N(0,I)$ and $\Pi_{\mathcal B_l^d}$ is Euclidean projection. We initialize the lean adjoint with $\tilde g_1=-\nabla_{w_1}Q_{\hat\phi}^{\LatentPolicy{}}(s,w_1)/\alpha$. Writing $b_i(w)=2\vNToLBase(w,t_{i+1})-w/t_{i+1}$, the backward update is
\begin{equation}\label{eq: lean-adjoint-EM}
    \tilde g_{t_i}\gets\tilde g_{t_{i+1}}+\Delta t\,\mathrm{VJP}\left(\nabla_w b_i(w_{t_i}),\tilde g_{t_{i+1}}\right),\qquad i=N-1,\dots,0,
\end{equation}
where $\mathrm{VJP}(J,g)=J^\top g$. Regression uses only $t_0,\dots,t_{N-1}$, excluding $t_N=1$, with targets
\begin{equation}\label{eq: am-target-discrete}
    \hat v^\star_{t_i}=\operatorname{sg}\!\left(\vNToLBase(w_{t_i},t_i)-\frac{\bar\sigma_{t_i}^2}{2}\tilde g_{t_i}\right),\qquad i=0,\dots,N-1,
\end{equation}
where $\operatorname{sg}$ denotes stop-gradient. The sampled states and regression targets are held fixed when differentiating the loss in \Cref{alg: laser}, so gradients pass only through $v_\thetaNToL$.
To stabilize training, we clip terminal latent critic gradients before initializing $\tilde g_1$. For each minibatch, we compute the Euclidean norm of each sample's gradient $\nabla_{w_1}Q_{\hat\phi}^{\LatentPolicy{}}(s,w_1)$ and rescale gradients whose norms exceed the minibatch's 90th percentile so that their norms equal that threshold, while preserving their directions. We then divide all gradients by their mean norm, floored at $10^{-6}$, before multiplying by $-1/\alpha$. This operation leaves environment rewards unchanged and is separate from clipping network-parameter gradients.
For critic targets and evaluation, we numerically integrate the latent policy ODE and project its terminal output before decoding. Denoting the numerical ODE map by $\widehat{\mu}_{\thetaNToL}^{s}$, the deployed latent policy is $(\Pi_{\mathcal B_l^d}\circ\widehat{\mu}_{\thetaNToL}^{s})_\sharp\BaseDist{}$. Projection enforces latent support, but if it places positive probability on the boundary, the resulting law has a singular component and no Lebesgue density. Thus, the entropy identity \eqref{eq: flow-entropy} and the optimality characterization in \Cref{thm: ng-target-optimal} do not directly apply to this projected law. The projected implementation is a practical approximation to the density-based formulation.

\section{Details of the Reach-avoid Example in \Cref{fig: point-pos}}\label{app: point-pos-detail}

The agent starts in one of two regions at the bottom and must reach the goal at the top. Actions are two-dimensional velocities. If a proposed next state lies inside an obstacle, we set $s' = s$. The ``No Entropy Regularization'' baseline is constructed by replacing the adjoint matching component in \algo{} with plain gradient ascent, and the generated latent samples $z$ are squashed inside the ball $\mathcal B_l^d$ to satisfy the support constraint. The ``No Support Constraint'' baseline is constructed with a QAM \citep{li2026qam} model with a large inverse temperature $\tau$ that drives the action distribution away from the behavior policy distribution.

\section{Experiments}\label{app: experiments}

\subsection{Environments}

We conduct experiments on the widely used OGBench offline RL benchmark \citep{ogbench_park2025}. Although OGBench is designed for offline goal-conditioned RL, it provides the \envname{-singletask} flag for standard offline RL. We use 4 environments in the main comparison, including one locomotion environment (\envname{antmaze}) and three manipulation environments (\envname{cube-single}, \envname{cube-double}, and \envname{scene}), each with 5 tasks. For each task, we consider 2 datasets, yielding a total of 40 tasks.

We consider the following tasks with \datasetname{clean} datasets, i.e., the demonstrations are randomly generated by an expert policy:
\begin{itemize}[leftmargin=1cm]
    \item \envname{antmaze-large-navigate-singletask-task\{1, 2, 3, 4, 5\}-v0}
    \item \envname{cube-single-play-singletask-task\{1, 2, 3, 4, 5\}-v0}
    \item \envname{cube-double-play-singletask-task\{1, 2, 3, 4, 5\}-v0}
    \item \envname{scene-play-singletask-task\{1, 2, 3, 4, 5\}-v0}
\end{itemize}
In addition, we consider the following tasks with \datasetname{noisy} datasets, i.e., the demonstrations are generated by a highly noisy policy:
\begin{itemize}[leftmargin=1cm]
    \item \envname{antmaze-large-explore-singletask-task\{1, 2, 3, 4, 5\}-v0}
    \item \envname{cube-single-noisy-singletask-task\{1, 2, 3, 4, 5\}-v0}
    \item \envname{cube-double-noisy-singletask-task\{1, 2, 3, 4, 5\}-v0}
    \item \envname{scene-noisy-singletask-task\{1, 2, 3, 4, 5\}-v0}
\end{itemize}
More details of the environments and datasets can be found in \cite{ogbench_park2025}.

\subsection{Baselines}

In this section, we briefly introduce the details of each baseline.

\paragraph{\baselinename{FQL}.}
\baselinename{FQL} \citep{fql_park2025} is an algorithm that uses the statistical distance regularization objective \eqref{eq: statistical-distance-reg-loss}, using a squared action-matching penalty that upper-bounds the squared 2-Wasserstein distance to its learned BC policy. It begins by training the decoder $\muLToACondS{}$ via the flow-matching loss \eqref{eq: flow-matching}, assuming a standard normal latent distribution $\LatentBCDist{} = \mathcal N(0, I)$. Subsequently, it fine-tunes this decoder by learning a \textit{one-step} distilled policy $\mu_\thetaLToAPrime(\pl z, \pl s): \LatentSpace{}\times\mathcal S\to\ActionSpace{}$. This policy is optimized to maximize expected return while remaining anchored to the behavior cloning (BC) decoder $\muLToACondS{}$. Specifically, the one-step policy minimizes the following loss:
\begin{equation}
    \begin{aligned}
        L(\thetaLToAPrime) &= \mathbb E_{s\sim D, z\sim \LatentBCDist{}}\left[-Q^{\mu_\thetaLToAPrime}_\phi(s, \mu_\thetaLToAPrime(z, s))\right] \\
        &\quad + \alpha \mathbb E_{s\sim D, z\sim \LatentBCDist{}}\left[\|\mu_\thetaLToAPrime(z, s) - \muLToACondS(z)\|^{2}\right],
    \end{aligned}
\end{equation}
where the first term optimizes task performance, and the second term is an upper-bound surrogate for their squared 2-Wasserstein distance \citep{fql_park2025}. Although the one-step policy is less expressive than the original flow policy \citep{sorl}, training this policy to maximize the $Q$-value avoids the BPTT process. However, the Wasserstein distance regularization is not enough to avoid OOD actions \citep{zhang2026reform}.

\paragraph{\baselinename{QAM}.}
\baselinename{QAM} \citep{li2026qam} is an algorithm that uses the statistical distance regularization objective \eqref{eq: statistical-distance-reg-loss}, instantiating $D(\pi_\theta\;\|\;\BehaviorPol{})$ as the KL divergence. It begins by training the decoder $\muLToACondS{}$ via the flow-matching loss \eqref{eq: flow-matching}, assuming a standard normal latent distribution $\LatentBCDist{} = \mathcal N(0, I)$. Subsequently, it fine-tunes the decoder via adjoint matching to align with the target tilted distribution $\pi_\beta(\cdot|s)\exp(\tau Q^\muLToA(s, \cdot))$. This procedure is mathematically equivalent to optimizing the following KL-divergence-regularized objective:
\begin{equation}
    L(\theta) = -\mathbb E_{s\sim \mathcal{D}} \left[ \mathbb E_{a\sim \pi_\theta(\cdot|s)}\left[Q^\muLToA(s, a)\right] - \frac{1}{\tau} D_{\mathrm{KL}}(\pi_\theta(\cdot|s) \;\|\; \pi_\beta(\cdot|s)) \right],
\end{equation}
where $\pi_\theta(\cdot|s) = (\mu_\thetaLToAPrime^s)_\sharp \LatentBCDist{}$, and $\mu_\thetaLToAPrime^s$ is the fine-tuned flow decoder. \baselinename{QAM} avoids the BPTT process by leveraging the adjoint matching technique, similar to \ourname{}, but the KL divergence regularization can limit the policy improvement \citep{zhang2026reform}.

\paragraph{\baselinename{QAM-E}.}
\baselinename{QAM-E} \citep{li2026qam} extends \baselinename{QAM} with an edit policy. After fine-tuning the decoder $\mu_\thetaLToAPrime^s$, it learns an edit policy $\pi_\omega(\pl a|\pl s, \pl a): \mathcal S\times\ActionSpace{}\to\Delta(\ActionSpace{})$ that modifies the action generated by $\pi_\theta(\cdot|s) = (\mu_\thetaLToAPrime^s)_\sharp \LatentBCDist{}$ by at most $\sigma_a$ to further improve performance. In practice, the edit policy $\pi_\omega$ is parameterized as a squashed Gaussian policy that only outputs values in $[-\sigma_a, \sigma_a]$. Empirically, \baselinename{QAM-E} outperforms \baselinename{QAM} \citep{li2026qam}, but it is sensitive to $\sigma_a$, and the edit policy distribution may generate OOD actions.

\paragraph{\baselinename{IFQL}.}
Presented in \citep{fql_park2025}, \baselinename{IFQL} is the flow version of the implicit diffusion Q-learning (IDQL) algorithm \citep{hansen2023idql}. Similarly to the aforementioned algorithms, it first learns the decoder $\muLToACondS{}$ via the flow-matching loss \eqref{eq: flow-matching}, assuming $\LatentBCDist{} = \mathcal N(0, I)$. However, instead of learning the policy $Q$-function with an actor-critic structure with loss \eqref{eq: critic-loss}, it learns an approximately optimal critic via implicit Q-learning (IQL) \citep{kostrikov2022offline} using the following losses:
\begin{align}
    L(\phi) &= \mathbb E_{(s, a, s')\sim\mathcal D}\left[(Q_\phi(s, a) - r - \gamma V_\xi(s'))^2\right], \\
    L(\xi) &= \mathbb E_{(s, a)\sim\mathcal D}\left[L_2^\kappa(Q_{\hat\phi}(s, a) - V_\xi(s))\right],
\end{align}
where $L_2^\kappa(u) = |\kappa - \mathbbm{1}(u < 0)|u^2$ is the expectile regression loss. We use $\kappa=0.9$ for all tasks. The bootstrap term $\gamma V_\xi(s')$ is masked out for terminal transitions. To extract a policy from the learned $Q$-function and the decoder, \baselinename{IFQL} uses best-of-$N$ sampling with $N=32$ candidates:
\begin{equation}
    \begin{gathered}
    i^\star=\argmax_{i\in\{1,\dots,N\}} Q(s,a_i),\qquad a^\star=a_{i^\star},\\
    a_1,\dots,a_N\overset{\mathrm{i.i.d.}}{\sim}(\muLToACondS)_\sharp\LatentBCDist{}.
    \end{gathered}
\end{equation}
\baselinename{IFQL} satisfies the support constraint by construction, but the candidate search is inefficient in high-dimensional spaces with a complex $Q$-function.

\paragraph{\baselinename{DSRL}.}
\baselinename{DSRL} \citep{wagenmaker2025steering} is a recently proposed latent-space RL algorithm, which has several key differences compared with \ourname{}. First, the latent distribution of the decoder is a standard Gaussian in \baselinename{DSRL}, i.e., $\LatentBCDist{} = \mathcal N(0, I)$. Because the support of $\LatentBCDist{}$ is unbounded, steering $\LatentBCDist{}$ might generate a distribution that has high density far from the origin, where the samples are not well represented during the training of the decoder. To avoid this problem, \baselinename{DSRL} parameterizes the latent policy $\LatentPolicy{}$ with a squashed Gaussian distribution that only generates samples within the range of $[-\sigma_z, \sigma_z]$, where $\sigma_z$ is a hyperparameter that requires fine-tuning for each task. \baselinename{DSRL} follows an SAC-style loss \citep{haarnoja2018soft} and explicitly regularizes the entropy, which is straightforward for a Gaussian latent policy. However, because of the Gaussian parameterization, \baselinename{DSRL} has a suboptimal entropy-performance tradeoff as shown in \Cref{fig: ent-q-tradeoff}.

\paragraph{\baselinename{ReFORM}.}
\baselinename{ReFORM} \citep{zhang2026reform} is another recently proposed latent-space RL algorithm, which also uses a flow latent policy. \baselinename{ReFORM} learns its BC decoder $\muLToACondS$ in the same way as \ourname{}.
However, unlike \ourname{}, \baselinename{ReFORM} does not train the latent policy via adjoint matching, but via plain gradient descent. In addition, \baselinename{ReFORM} does not have the explicit entropy regularization term, resulting in the following loss:
\begin{equation}
    L_{\NToL}(\thetaNToL) = -\mathbb E_{s\sim\mathcal D, w\sim \BaseDist{}}\left[Q^{\LatentPolicy{}}(s, \muNToL(w; s))\right].
\end{equation}
Directly using gradient descent on this loss cannot guarantee satisfaction of the latent-space support constraint \eqref{eq: latent-support-constraint}, as the flow $\muNToL{}$ can drive the samples outside the closed ball $\mathcal B_l^d$. To address this issue, \baselinename{ReFORM} modifies the Euler integration step to \eqref{eq: reform-modified-euler-step} when integrating the flow $\muNToL{}$, and proves that this modified step guarantees satisfaction of the latent-space support constraint. However, as discussed in \Cref{app: reform-issue}, this modified step
can cause ``static flow'' when a proposal leaves the ball and the velocity aligns with the outward normal. In our toy example, this acts as implicit entropy regularization but yields a suboptimal action distribution.

\subsection{Hyperparameters}

We organize the hyperparameters into three distinct categories. \textbf{Fixed common hyperparameters} (\Cref{tab: common-hyperparams}) are shared across all algorithms and remain constant across all tasks. We do not fine-tune these, but instead adopt the standard values established in prior work \citep{fql_park2025,li2026qam,zhang2026reform}. \textbf{Environment-specific common hyperparameters} (\Cref{tab: env-specific-common-hyperparams}) are also shared across algorithms but vary by environment and dataset to optimize baseline performance. Following \cite{li2026qam}, we perform a minimal sweep over two candidate values for these settings: 1 or 5 for the action chunking length, and 0 or 0.5 for the pessimistic backup weight. \textbf{Environment-specific unique hyperparameters} (\Cref{tab: env-specific-unique-hyperparams}) are specific to individual baselines (namely \baselinename{FQL}, \baselinename{DSRL}, \baselinename{QAM}, and \baselinename{QAM-E}) and require per-environment tuning. For these, we rely on the optimal values reported in their respective papers or tune them empirically. One notable exception is the edit scale $\sigma_a$ for \baselinename{QAM-E}. While the original authors sometimes set $\sigma_a = 0$ \citep{li2026qam}, making \baselinename{QAM-E} identical to \baselinename{QAM}, we enforce a minimum value of $\sigma_a = 0.01$ to ensure the algorithms remain distinct. Finally, while \ourname{} introduces two unique hyperparameters, we keep them strictly fixed across all tasks. Specifically, we set $1/\alpha=10$ and choose $\sigma_0$ such that $\mathrm{Pr}_{w\sim\mathcal N(0, \sigma_0^2 I)}(\|w\|\leq l) = 0.999$.
For \ourname{} and \baselinename{ReFORM}, the latent ball radius is $l=\sqrt d$, where $d$ is the action or action chunk dimension \citep{zhang2026reform}.

\begin{table}[t]
    \centering
    \caption{\textbf{Fixed common hyperparameters.} These hyperparameters are shared by \ourname{} and all the baselines and are fixed in all environments.}
    \label{tab: common-hyperparams}
    \begin{tabular}{l|l}
        \toprule
        Hyperparameter & Value \\
        \midrule
        Learning rate & 0.0003 \\
        Optimizer & Adam \citep{adam} \\
        Maximum gradient norm & 1 (\baselinename{QAM}, \baselinename{QAM-E}), 10 (others) \\
        Target network smoothing coefficient & $0.005$ \\
        Discount factor $\gamma$ & $0.995$ \\
        MLP dimensions & $[512, 512, 512, 512]$ \\
        Nonlinearity & GELU \citep{hendrycks2016gelu} \\
        Flow integration steps & $10$ \\
        Flow time sampling distribution & $\mathcal U[0, 1]$ \\
        Minibatch size & $256$ \\
        Number of critics in the ensemble & 8 \\
        $Q$-function aggregation method & mean \\
        Actor layer normalization \citep{ba2016layer} & False \\
        Critic layer normalization \citep{ba2016layer} & True \\
        \bottomrule
    \end{tabular}
\end{table}

\begin{table}[t]
    \centering
    \caption{\textbf{Environment-specific common hyperparameters.} These hyperparameters are shared by \ourname{} and all baselines, but vary by environment and dataset.}
    \label{tab: env-specific-common-hyperparams}
    \begin{tabular}{l|c|cccc}
        \toprule
        Environment & Dataset & Action chunk length & Pessimistic backup weight $\rho$ \\
        \midrule
        \envname{antmaze-large-navigate} & \datasetname{clean} & 1 & 0.5 \\
        \envname{antmaze-large-explore} & \datasetname{noisy} & 1 & 0.5 \\
        \envname{cube-single-play} & \datasetname{clean} & 5 & 0.5 \\
        \envname{cube-single-noisy} & \datasetname{noisy} & 1 & 0 \\
        \envname{cube-double-play} & \datasetname{clean} & 5 & 0.5 \\
        \envname{cube-double-noisy} & \datasetname{noisy} & 1 & 0 \\
        \envname{scene-play} & \datasetname{clean} & 5 & 0.5 \\
        \envname{scene-noisy} & \datasetname{noisy} & 5 & 0.5 \\
        {\envname{puzzle-4x4-play}} & {\datasetname{clean}} & {5} & {0.5} \\
        \bottomrule
    \end{tabular}
\end{table}

\begin{table}[t]
    \centering
    \caption{\textbf{Environment- and dataset-specific unique hyperparameters.} These hyperparameters are specific to individual algorithms and are tuned by environment and dataset. The method-specific hyperparameters of \ourname{}, \baselinename{ReFORM}, and \baselinename{IFQL} remain fixed across environments, while their common settings follow \Cref{tab: env-specific-common-hyperparams}.}
    \label{tab: env-specific-unique-hyperparams}
    \begin{tabular}{l|c|cccc}
        \toprule
        Environment & Dataset & \baselinename{FQL} ($\alpha$) & \baselinename{DSRL} ($\sigma_z$) & \baselinename{QAM} ($\tau$) & \baselinename{QAM-E} ($\tau, \sigma_a$) \\
        \midrule
        \envname{antmaze-large-navigate} & \datasetname{clean} & 3.0 & 1.25 & 10.0 & (1.0, 0.1) \\
        \envname{antmaze-large-explore} & \datasetname{noisy} & 0.3 & 1.25 & 10.0 & (1.0, 0.1) \\
        \envname{cube-single-play} & \datasetname{clean} & 300 & 0.5 & 1.0 & (1.0, 0.01) \\
        \envname{cube-single-noisy} & \datasetname{noisy} & 30 & 0.5 & 1.0 & (3.0, 0.5) \\
        \envname{cube-double-play} & \datasetname{clean} & 300 & 1.5 & 1.0 & (1.0, 0.01) \\
        \envname{cube-double-noisy} & \datasetname{noisy} & 30 & 1.5 & 1.0 & (1.0, 0.01) \\
        \envname{scene-play} & \datasetname{clean} & 300 & 0.75 & 1.0 & (1.0, 0.01) \\
        \envname{scene-noisy} & \datasetname{noisy} & 30 & 0.75 & 1.0 & (1.0, 0.01) \\
        {\envname{puzzle-4x4-play}} & {\datasetname{clean}} & {1.0} & {1.0} & {30.0} & {(0.1, 0.9)} \\
        \bottomrule
    \end{tabular}
\end{table}

\subsection{Computational Resources}\label{app: computation-resources}

Most of our experiments are run on a 13th Gen Intel(R) Core(TM) i7-13700KF CPU with 64 GB of RAM and an NVIDIA GeForce RTX 4090 GPU.
\Cref{tab: training-time} reports training throughput. For a fixed iteration budget, \ourname{} requires about $12\%$ more training time than \baselinename{ReFORM} and $33\%$ more than \baselinename{QAM}, while its throughput is comparable to \baselinename{QAM-E} ($123$ vs $125\,\mathrm{it/s}$). The additional work includes computing lean adjoints and learning a latent policy alongside the decoder. Note that the lean-adjoint solve is needed only during offline training. At deployment, actions are generated by integrating the learned latent flow and then the decoder, without adjoints or critic gradients.

\begin{table}[t]
    \centering
    \caption{\textbf{Training throughput comparison.} We report the number of iterations per second (it/s) for \ourname{} and all the baselines.}
    \label{tab: training-time}
    \begin{tabular}{c|ccccccc}
        \toprule
        Algorithm & \baselinename{FQL} & \baselinename{QAM} & \baselinename{QAM-E} & \baselinename{IFQL} & \baselinename{DSRL} & \baselinename{ReFORM} & \ourname{} \\
        \midrule
        it/s & 220 & 163 & 125 & 210 & 165 & 138 & 123 \\
        \bottomrule
    \end{tabular}
\end{table}

\subsection{Additional Results}\label{app: experiment-additional-results}

\paragraph{Ablation study of entropy regularization.}
To test the efficacy of entropy regularization beyond the simple environment introduced in \Cref{sec: er-essential}, we extend the entropy regularization comparison to \envname{cube-single} and \envname{cube-double} environments with both \datasetname{clean} and \datasetname{noisy} datasets. \baselinename{NoER} uses the same latent support constraint, network architectures, critic design, and other training settings as \ourname{}, but directly maximizes $Q$ through the latent flow using reparametrized gradients without entropy regularization. \baselinename{ReFORM} instead uses a reflected flow with implicit entropy regularization (\Cref{app: reform-issue}), while \ourname{} optimizes an explicitly entropy-regularized objective via adjoint matching. \Cref{tab: entropy-ablation} reports the results for all five tasks in each environment with both \datasetname{clean} and \datasetname{noisy} datasets. We observe that both implicit (\baselinename{ReFORM}) and explicit (\ourname{}) entropy regularization help and significantly improve on \baselinename{NoER}, and explicit entropy regularization (\ourname{}) performs better than implicit entropy regularization (\baselinename{ReFORM}).

\begin{table}[!htbp]
    \begin{center}
    \caption{\textbf{Entropy regularization ablation on 20 \envname{cube-single} and \envname{cube-double} tasks.} Success rates (\%) are reported as mean $\pm$ standard deviation across 3 seeds, with 50 evaluation runs per seed. Bold denotes results at or above $95\%$ of the best mean in each setting.}
    \label{tab: entropy-ablation}
    \begin{tabular}{l|c|cc|c}
        \toprule
        Task & Dataset & \baselinename{NoER} & \baselinename{ReFORM} & \ourname{} \\
        \midrule
        \envname{cube-single-play-singletask-task1-v0} & \datasetname{clean} & $69${\tiny $\pm27$} & $\mathbf{99}${\tiny $\pm2$} & $\mathbf{100}${\tiny $\pm0$} \\
        \envname{cube-single-play-singletask-task2-v0} & \datasetname{clean} & $59${\tiny $\pm36$} & $\mathbf{96}${\tiny $\pm3$} & $\mathbf{100}${\tiny $\pm0$} \\
        \envname{cube-single-play-singletask-task3-v0} & \datasetname{clean} & $87${\tiny $\pm5$} & $\mathbf{98}${\tiny $\pm3$} & $\mathbf{100}${\tiny $\pm0$} \\
        \envname{cube-single-play-singletask-task4-v0} & \datasetname{clean} & $51${\tiny $\pm18$} & $87${\tiny $\pm2$} & $\mathbf{99}${\tiny $\pm1$} \\
        \envname{cube-single-play-singletask-task5-v0} & \datasetname{clean} & $74${\tiny $\pm14$} & $86${\tiny $\pm13$} & $\mathbf{93}${\tiny $\pm5$} \\
        \midrule
        \envname{cube-single-noisy-singletask-task1-v0} & \datasetname{noisy} & $76${\tiny $\pm6$} & $\mathbf{100}${\tiny $\pm0$} & $\mathbf{100}${\tiny $\pm0$} \\
        \envname{cube-single-noisy-singletask-task2-v0} & \datasetname{noisy} & $86${\tiny $\pm3$} & $\mathbf{100}${\tiny $\pm0$} & $\mathbf{100}${\tiny $\pm0$} \\
        \envname{cube-single-noisy-singletask-task3-v0} & \datasetname{noisy} & $91${\tiny $\pm7$} & $\mathbf{100}${\tiny $\pm0$} & $\mathbf{100}${\tiny $\pm0$} \\
        \envname{cube-single-noisy-singletask-task4-v0} & \datasetname{noisy} & $45${\tiny $\pm15$} & $\mathbf{100}${\tiny $\pm0$} & $\mathbf{100}${\tiny $\pm0$} \\
        \envname{cube-single-noisy-singletask-task5-v0} & \datasetname{noisy} & $59${\tiny $\pm7$} & $\mathbf{100}${\tiny $\pm0$} & $\mathbf{100}${\tiny $\pm0$} \\
        \midrule
        \envname{cube-double-play-singletask-task1-v0} & \datasetname{clean} & $73${\tiny $\pm29$} & $91${\tiny $\pm7$} & $\mathbf{97}${\tiny $\pm1$} \\
        \envname{cube-double-play-singletask-task2-v0} & \datasetname{clean} & $55${\tiny $\pm7$} & $83${\tiny $\pm7$} & $\mathbf{89}${\tiny $\pm1$} \\
        \envname{cube-double-play-singletask-task3-v0} & \datasetname{clean} & $63${\tiny $\pm12$} & $73${\tiny $\pm7$} & $\mathbf{81}${\tiny $\pm7$} \\
        \envname{cube-double-play-singletask-task4-v0} & \datasetname{clean} & $29${\tiny $\pm10$} & $22${\tiny $\pm4$} & $\mathbf{49}${\tiny $\pm12$} \\
        \envname{cube-double-play-singletask-task5-v0} & \datasetname{clean} & $36${\tiny $\pm21$} & $54${\tiny $\pm13$} & $\mathbf{70}${\tiny $\pm6$} \\
        \midrule
        \envname{cube-double-noisy-singletask-task1-v0} & \datasetname{noisy} & $61${\tiny $\pm15$} & $80${\tiny $\pm3$} & $\mathbf{92}${\tiny $\pm2$} \\
        \envname{cube-double-noisy-singletask-task2-v0} & \datasetname{noisy} & $51${\tiny $\pm12$} & $29${\tiny $\pm14$} & $\mathbf{81}${\tiny $\pm10$} \\
        \envname{cube-double-noisy-singletask-task3-v0} & \datasetname{noisy} & $27${\tiny $\pm5$} & $29${\tiny $\pm3$} & $\mathbf{61}${\tiny $\pm12$} \\
        \envname{cube-double-noisy-singletask-task4-v0} & \datasetname{noisy} & $50${\tiny $\pm3$} & $9${\tiny $\pm2$} & $\mathbf{65}${\tiny $\pm5$} \\
        \envname{cube-double-noisy-singletask-task5-v0} & \datasetname{noisy} & $27${\tiny $\pm6$} & $11${\tiny $\pm6$} & $\mathbf{38}${\tiny $\pm13$} \\
        \bottomrule
    \end{tabular}
    \end{center}
\end{table}

\paragraph{Inverse temperature and target policy smoothing.}
We compare $Q$-only latent optimization (\baselinename{NoER}), target policy smoothing (\baselinename{TPS}), and \ourname{} with different inverse temperatures on \envname{cube-double} task~2 using both \datasetname{clean} and \datasetname{noisy} datasets. \baselinename{TPS} \citep{fujimoto2018addressing} augments \baselinename{NoER} by adding independent zero-mean Gaussian noise with standard deviation $0.2$, clipped to $[-0.5,0.5]$, to the target actions in the critic update. The perturbed actions are then clipped to the action bounds. As shown in \Cref{tab: inverse-temperature-tps}, our default $1/\alpha=10$ outperforms both \baselinename{NoER} and \baselinename{TPS} on both datasets. Strong regularization ($1/\alpha=3$) reduces performance on both datasets, whereas weak regularization ($1/\alpha=30$) performs poorly on the \datasetname{clean} dataset but remains competitive on the \datasetname{noisy} dataset.
One possible explanation is that noisy demonstrations provide broader action coverage, reducing critic artifacts.

\begin{table}[!htbp]
    \begin{center}
    \caption{\textbf{Comparison of \baselinename{NoER}, \baselinename{TPS}, and \ourname{} with different inverse temperatures on \envname{cube-double} task 2.} Success rates (\%) are reported as mean $\pm$ standard deviation across 3 seeds, with 50 evaluation runs per seed. Bold denotes results at or above $95\%$ of the best mean for each dataset.}
    \label{tab: inverse-temperature-tps}
    \begin{tabular}{l|cc|ccccc}
        \toprule
        \multirow{2}{*}[-2pt]{Dataset} & \multirow{2}{*}[-2pt]{\baselinename{NoER}} & \multirow{2}{*}[-2pt]{\baselinename{TPS}} & \multicolumn{5}{c}{\ourname{} ($1/\alpha$)} \\[2pt]
        \cline{4-8}
        & & & \rule{0pt}{\dimexpr\ht\strutbox+2pt\relax}$1/\alpha = 3$ & $1/\alpha = 5$ & $1/\alpha = 10$ & $1/\alpha = 20$ & $1/\alpha = 30$ \\
        \midrule
        \datasetname{clean} & $55${\tiny $\pm7$} & $49${\tiny $\pm24$} & $50${\tiny $\pm14$} & $78${\tiny $\pm19$} & $\mathbf{89}${\tiny $\pm1$} & $79${\tiny $\pm12$} & $30${\tiny $\pm42$} \\
        \datasetname{noisy} & $51${\tiny $\pm12$} & $61${\tiny $\pm17$} & $31${\tiny $\pm19$} & $72${\tiny $\pm15$} & $\mathbf{81}${\tiny $\pm10$} & $\mathbf{83}${\tiny $\pm18$} & $\mathbf{83}${\tiny $\pm19$} \\
        \bottomrule
    \end{tabular}
    \end{center}
\end{table}

\paragraph{Uniform-mixture ablation in action space.}
\ourname{} combines entropy optimization with bounded latent support. To examine a simple alternative in action space, we evaluate \baselinename{QAM+Mix}, which broadens QAM's reference toward a uniform distribution over the action space. With probability $\eta$, each behavior-cloning target is replaced by an action sampled uniformly within $\mathcal A$. Adjoint matching then uses the reference flow learned from these modified targets.
In the ideal KL-regularized formulation, $\eta=0$ recovers QAM, while $\eta=1$ yields action-space maximum-entropy optimization.
Thus, $\eta$ controls how strongly the reference distribution is broadened toward uniform.
On \envname{cube-double-play-singletask-task2-v0} with the \datasetname{clean} dataset, both tested mixing probabilities ($\eta=0.03$ and $\eta=0.1$) reduce \baselinename{QAM}'s mean success rate (\Cref{tab: action-space-entropy}).
The results show that this na\"ive approach does not perform well.

\begin{table}[!htbp]
    \begin{center}
    \caption{\textbf{Uniform-mixture ablation in action space} on \envname{cube-double} task 2 with the \datasetname{clean} dataset. Success rates (\%) are reported as mean $\pm$ standard deviation across 3 seeds, with 50 evaluation runs per seed. Bold denotes results at or above $95\%$ of the best mean.}
    \label{tab: action-space-entropy}
    \begin{tabular}{c|cc|c}
        \toprule
        \baselinename{QAM} & \baselinename{QAM+Mix} ($\eta=0.03$) & \baselinename{QAM+Mix} ($\eta=0.1$) & \ourname{} \\
        \midrule
        $83${\tiny $\pm11$} & $38${\tiny $\pm10$} & $9${\tiny $\pm6$} & $\mathbf{89}${\tiny $\pm1$} \\
        \bottomrule
    \end{tabular}
    \end{center}
\end{table}

\begin{samepage}
\paragraph{Evaluation on harder tasks.}
We compare \ourname{} with all six baselines on the five \envname{puzzle-4x4} tasks using the \datasetname{clean} (100M) dataset. Every run uses two million updates, and we evaluate the final checkpoint for each method and seed. The action chunk length is 5, giving a 25-dimensional action chunk, and the pessimistic backup weight is $\rho=0.5$. Baseline-specific settings are listed in \Cref{tab: env-specific-unique-hyperparams}. As shown in \Cref{tab: puzzle-4x4}, \ourname{} still achieves high performance on these more challenging tasks.

\end{samepage}

\begin{table}[!htbp]
    \begin{center}
    \caption{\textbf{Results on the five \envname{puzzle-4x4} tasks with the \datasetname{clean} dataset (100M).} Success rates (\%) are reported as mean $\pm$ standard deviation across 3 seeds, with 50 evaluation runs per seed. Bold denotes results at or above $95\%$ of the best mean for each task.}
    \label{tab: puzzle-4x4}
    \begin{tabular}{c|cccccc|c}
        \toprule
        Task & \baselinename{FQL} & \baselinename{QAM} & \baselinename{QAM-E} & \baselinename{IFQL} & \baselinename{DSRL} & \baselinename{ReFORM} & \ourname{} \\
        \midrule
        1 & $5${\tiny $\pm4$} & $2${\tiny $\pm2$} & $63${\tiny $\pm14$} & $50${\tiny $\pm13$} & $65${\tiny $\pm24$} & $86${\tiny $\pm6$} & $\mathbf{97}${\tiny $\pm3$} \\
        2 & $13${\tiny $\pm9$} & $0${\tiny $\pm0$} & $74${\tiny $\pm24$} & $2${\tiny $\pm0$} & $17${\tiny $\pm7$} & $6${\tiny $\pm7$} & $\mathbf{87}${\tiny $\pm2$} \\
        3 & $5${\tiny $\pm7$} & $3${\tiny $\pm2$} & $18${\tiny $\pm12$} & $\mathbf{87}${\tiny $\pm1$} & $63${\tiny $\pm41$} & $81${\tiny $\pm7$} & $\mathbf{84}${\tiny $\pm7$} \\
        4 & $1${\tiny $\pm1$} & $0${\tiny $\pm0$} & $33${\tiny $\pm15$} & $32${\tiny $\pm4$} & $31${\tiny $\pm21$} & $33${\tiny $\pm25$} & $\mathbf{60}${\tiny $\pm3$} \\
        5 & $0${\tiny $\pm0$} & $1${\tiny $\pm1$} & $53${\tiny $\pm26$} & $2${\tiny $\pm3$} & $14${\tiny $\pm14$} & $6${\tiny $\pm3$} & $\mathbf{69}${\tiny $\pm18$} \\
        \bottomrule
    \end{tabular}
    \end{center}
\end{table}

\paragraph{Boundary-hit diagnostic for \baselinename{ReFORM}.}
We evaluate the final \baselinename{ReFORM} checkpoints on all five \envname{cube-double} tasks with both \datasetname{clean} and \datasetname{noisy} datasets. At each policy call, we record the fraction of $N=10$ latent Euler proposals satisfying $\|\hat w_{k+1}\|>l$, before applying the reflection correction in \eqref{eq: reform-modified-euler-step}. We average these rates over policy calls within each episode, then over episodes and seeds. \Cref{tab: reform-boundary-hits} reports the results.
The high rates ($97.4\%$--$99.0\%$) show that boundary correction is frequently active, consistent with the proposed static-flow mechanism in \Cref{app: reform-issue}.

\begin{table}[!htbp]
    \begin{center}
    \caption{\textbf{Boundary-hit rates (\%) for \baselinename{ReFORM} on \envname{cube-double}.} Results are averaged across 3 seeds, with 50 evaluation episodes per seed and 10 latent flow steps per policy call.}
    \label{tab: reform-boundary-hits}
    \begin{tabular}{l|ccccc}
        \toprule
        Dataset & Task 1 & Task 2 & Task 3 & Task 4 & Task 5 \\
        \midrule
        \datasetname{clean} & $98.2$ & $99.0$ & $98.9$ & $97.6$ & $98.0$ \\
        \datasetname{noisy} & $97.4$ & $97.7$ & $97.9$ & $97.5$ & $97.9$ \\
        \bottomrule
    \end{tabular}
    \end{center}
\end{table}

\begin{figure}[!htbp]
    \centering
    \includegraphics[width=0.32\linewidth]{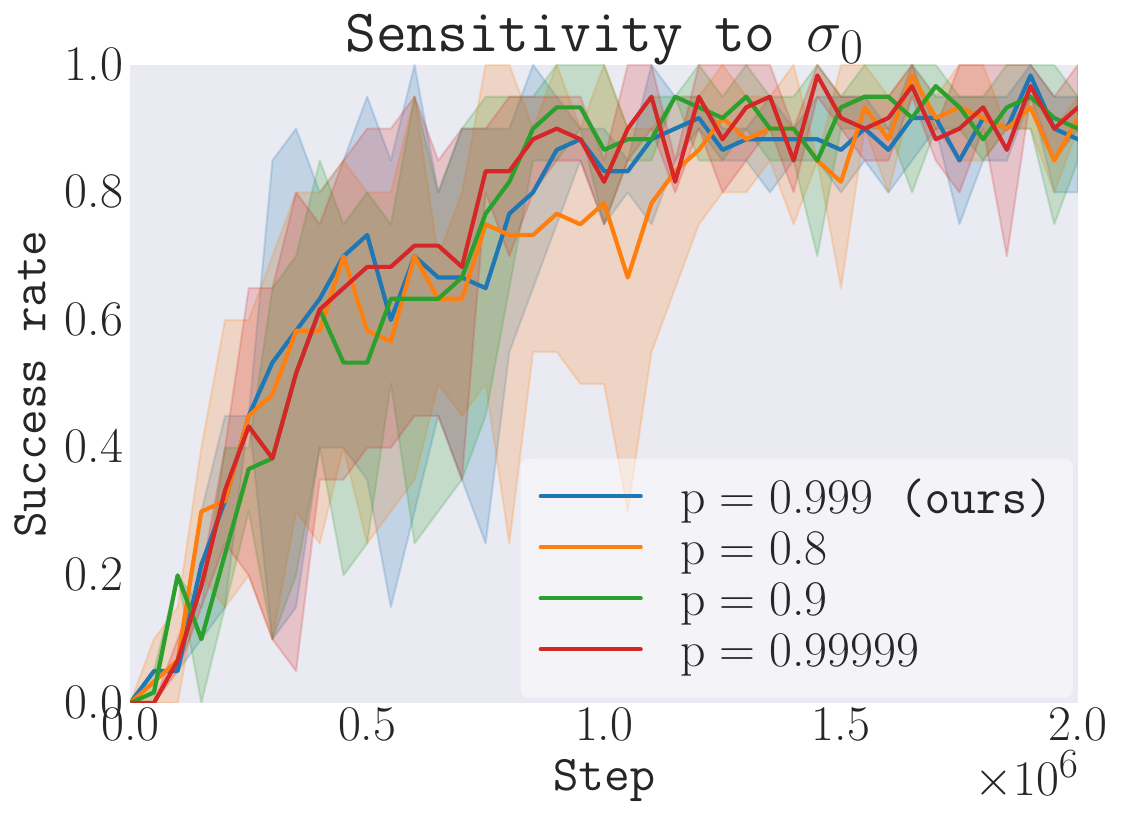}
    \includegraphics[width=0.32\linewidth]{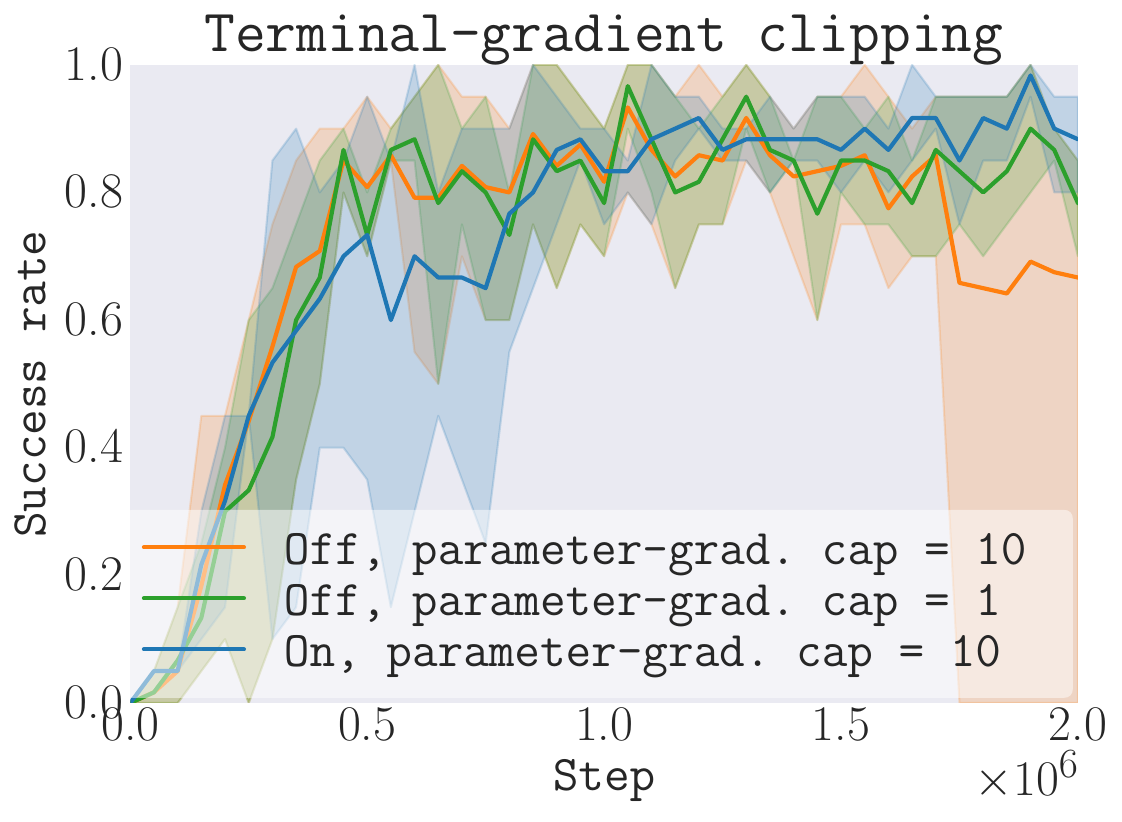}
    \caption{\textbf{Additional ablation studies.} \ourname{} demonstrates robustness to the choice of $\sigma_0$. Furthermore, applying terminal-gradient clipping improves training stability. Reducing the parameter-gradient cap from 10 to 1 yields a comparable stabilizing effect. Lines show mean success rates, and shading spans the minimum and maximum across runs.}
    \label{fig: ablations-additional}
\end{figure}

\paragraph{Alternative latent prior.}\label{app: latent-mixture}
For the decoder-prior ablation in \Cref{fig: ablations}, we use
\[
    q_{\mathsf L}^{\mathrm{mix}}=0.4\,\mathcal U(\mathcal B_l^d)
    +0.3\,(\Pi_{\mathcal B_l^d})_\sharp\mathcal N(0.5e_1,0.1^2I)
    +0.3\,(\Pi_{\mathcal B_l^d})_\sharp\mathcal N(0.4e_d,0.15^2I),
\]
where $e_i$ is the $i$th coordinate vector and Gaussian samples outside the ball are projected onto its boundary. This experiment uses $d=25$ and $l=5$. Only the decoder's training prior changes. The latent policy's entropy reference remains $\LatentUniform{}=\mathcal U(\mathcal B_l^d)$.

\paragraph{Ablation study of the standard deviation of the base distribution.}
We study the sensitivity of \algo{} w.r.t. the standard deviation $\sigma_0$ of $\BaseDist{}$. We vary $\sigma_0$ such that the ball $\mathcal B_l^d$ contains probability mass $p$ under the Gaussian distribution $\mathcal N(0, \sigma_0^2 I)$. We report the results in \Cref{fig: ablations-additional}. The results show that \algo{} is not sensitive to $\sigma_0$.

\paragraph{Ablation study of the gradient clipping technique.}
A known drawback of adjoint matching is its tendency to produce large gradients, which can destabilize the training process \citep{domingo2024adjoint,li2026qam}.
In \Cref{fig: ablations-additional}, we compare three variants on \envname{cube-double} task~2 with \datasetname{clean} data: terminal-gradient clipping disabled with a parameter-gradient cap of 10, disabled with a cap of 1, and enabled with a cap of 10 (our default). Both enabling terminal-gradient clipping and reducing the parameter-gradient cap improve stability over the first variant, with comparable performance.

\paragraph{Detailed results.}
The full results for each algorithm on the 40 main evaluation settings are provided in \Cref{tab: full-result}. Success rates are bolded for algorithms achieving at least $95\%$ of the best performance per task.
For every method and seed, we evaluate the final checkpoint. The main-suite training budgets are 5 million updates for \envname{antmaze-large-navigate}, 8 million for \envname{antmaze-large-explore}, 1 million for \envname{cube-single-noisy}, and 2 million for the remaining settings. The training budgets are selected such that all methods converge.
Each checkpoint is evaluated over 50 episodes, and we report the mean and standard deviation across 3 seeds.
We observe that \ourname{} yields the best overall performance, achieving near-perfect success rates on multiple tasks.

\paragraph{Statistical comparison.}
We pair results across the 40 main task--dataset settings, averaging each method's success rate over its 3 seeds before comparison. For each baseline $j$, a one-sided paired $t$-test evaluates $H_{0,j}:\Delta_j\leq0$, where $\Delta_j$ denotes the expected difference between \ourname{}'s and baseline $j$'s setting-level mean success rates. To test for improvement over every baseline, we use an intersection--union test of $H_0:\bigcup_j\{\Delta_j\leq0\}$ against $H_1:\Delta_j>0$ for all $j$. The global $p$-value is $\max_j p_j=1.35\times10^{-3}$ (\Cref{tab: main-statistics}), rejecting $H_0$ at the $0.05$ level under the paired-test assumptions. This inference treats setting-level differences as independent despite shared environments and datasets, so we interpret it as supporting evidence for aggregate performance on this suite.

\begin{table}[!htbp]
    \begin{center}
    \caption{\textbf{Paired statistical comparisons on the 40 main evaluation settings.} Mean gaps are \ourname{} minus each baseline in percentage points. One-sided $p$-values are computed from unrounded success rates averaged over 3 seeds per setting, with 50 evaluation runs per seed.}
    \label{tab: main-statistics}
    \setlength{\tabcolsep}{4pt}\begin{tabular}{l|cccccc}
        \toprule
        Baseline & \baselinename{FQL} & \baselinename{QAM} & \baselinename{QAM-E} & \baselinename{IFQL} & \baselinename{DSRL} & \baselinename{ReFORM} \\
        \midrule
        Mean gap & $20.02$ & $18.70$ & $17.75$ & $47.50$ & $16.33$ & $11.82$ \\
        $p$-value & $9.08\!\times\!10^{-6}$ & $3.47\!\times\!10^{-4}$ & $6.44\!\times\!10^{-4}$ & $3.22\!\times\!10^{-11}$ & $1.35\!\times\!10^{-3}$ & $3.42\!\times\!10^{-5}$ \\
        \bottomrule
    \end{tabular}
    \end{center}
\end{table}

\paragraph{Training curves.}
\Crefrange{fig: antmaze-large-navigate-training}{fig: scene-noisy-training} show training curves for all algorithms and tasks on both the \datasetname{clean} and \datasetname{noisy} datasets.
Lines show means over 3 seeds, and shading spans the minimum and maximum across seeds. Each point uses 20 evaluation episodes per seed, whereas the reported table results use a separate evaluation of 50 episodes per seed.

\subsection{Code}

We provide code on our project website \url{https://mit-realm.github.io/laser/}.

\begin{table}[!htbp]
	\begin{center}
    \caption{\textbf{Full results on 40 OGBench tasks.} Success rates (\%) are reported as mean $\pm$ standard deviation across 3 seeds, with 50 evaluation episodes per seed. Bold denotes results at or above $95\%$ of the best mean in each setting, following \cite{ogbench_park2025}. We omit the \envname{-singletask} tags from the task names to save space. }
    \label{tab: full-result}
    \scriptsize
    \begin{tabular}{l|c|cccccc|c}
        \toprule
        Task & Dataset & \baselinename{FQL} & \baselinename{QAM} & \baselinename{QAM-E} & \baselinename{IFQL} & \baselinename{DSRL} & \baselinename{ReFORM} & \ourname{} \\
        \midrule
        \envname{antmaze-large-navigate-task1-v0} & \datasetname{clean} & $64${\tiny $\pm45$} & $\mathbf{96}${\tiny $\pm3$} & $\mathbf{93}${\tiny $\pm2$} & $79${\tiny $\pm13$} & $90${\tiny $\pm3$} & $91${\tiny $\pm2$} & $\mathbf{95}${\tiny $\pm1$} \\
			\envname{antmaze-large-navigate-task2-v0} & \datasetname{clean} & $31${\tiny $\pm43$} & $60${\tiny $\pm42$} & $\mathbf{94}${\tiny $\pm3$} & $65${\tiny $\pm5$} & $81${\tiny $\pm5$} & $78${\tiny $\pm2$} & $\mathbf{89}${\tiny $\pm2$} \\
			\envname{antmaze-large-navigate-task3-v0} & \datasetname{clean} & $71${\tiny $\pm12$} & $85${\tiny $\pm3$} & $\mathbf{97}${\tiny $\pm1$} & $85${\tiny $\pm9$} & $\mathbf{94}${\tiny $\pm3$} & $74${\tiny $\pm6$} & $\mathbf{92}${\tiny $\pm3$} \\
			\envname{antmaze-large-navigate-task4-v0} & \datasetname{clean} & $33${\tiny $\pm41$} & $\mathbf{93}${\tiny $\pm1$} & $\mathbf{91}${\tiny $\pm1$} & $31${\tiny $\pm43$} & $75${\tiny $\pm14$} & $69${\tiny $\pm4$} & $85${\tiny $\pm1$} \\
			\envname{antmaze-large-navigate-task5-v0} & \datasetname{clean} & $88${\tiny $\pm6$} & $\mathbf{95}${\tiny $\pm1$} & $\mathbf{93}${\tiny $\pm4$} & $83${\tiny $\pm2$} & $\mathbf{91}${\tiny $\pm4$} & $83${\tiny $\pm1$} & $\mathbf{91}${\tiny $\pm1$} \\
			\midrule
			\envname{antmaze-large-explore-task1-v0} & \datasetname{noisy} & $\mathbf{97}${\tiny $\pm1$} & $0${\tiny $\pm0$} & $0${\tiny $\pm0$} & $50${\tiny $\pm36$} & $0${\tiny $\pm0$} & $91${\tiny $\pm4$} & $\mathbf{97}${\tiny $\pm2$} \\
			\envname{antmaze-large-explore-task2-v0} & \datasetname{noisy} & $\mathbf{97}${\tiny $\pm1$} & $0${\tiny $\pm0$} & $0${\tiny $\pm0$} & $3${\tiny $\pm2$} & $0${\tiny $\pm0$} & $88${\tiny $\pm9$} & $87${\tiny $\pm2$} \\
			\envname{antmaze-large-explore-task3-v0} & \datasetname{noisy} & $\mathbf{100}${\tiny $\pm0$} & $0${\tiny $\pm0$} & $0${\tiny $\pm0$} & $\mathbf{95}${\tiny $\pm1$} & $0${\tiny $\pm0$} & $\mathbf{100}${\tiny $\pm0$} & $\mathbf{98}${\tiny $\pm3$} \\
			\envname{antmaze-large-explore-task4-v0} & \datasetname{noisy} & $\mathbf{89}${\tiny $\pm3$} & $0${\tiny $\pm0$} & $0${\tiny $\pm0$} & $36${\tiny $\pm13$} & $0${\tiny $\pm0$} & $82${\tiny $\pm9$} & $\mathbf{89}${\tiny $\pm2$} \\
			\envname{antmaze-large-explore-task5-v0} & \datasetname{noisy} & $\mathbf{94}${\tiny $\pm6$} & $0${\tiny $\pm0$} & $0${\tiny $\pm0$} & $47${\tiny $\pm34$} & $0${\tiny $\pm0$} & $\mathbf{93}${\tiny $\pm2$} & $\mathbf{92}${\tiny $\pm3$} \\
			\midrule
			\envname{cube-single-play-task1-v0} & \datasetname{clean} & $\mathbf{95}${\tiny $\pm4$} & $94${\tiny $\pm3$} & $\mathbf{95}${\tiny $\pm4$} & $79${\tiny $\pm2$} & $95${\tiny $\pm5$} & $\mathbf{99}${\tiny $\pm2$} & $\mathbf{100}${\tiny $\pm0$} \\
			\envname{cube-single-play-task2-v0} & \datasetname{clean} & $\mathbf{97}${\tiny $\pm1$} & $\mathbf{99}${\tiny $\pm1$} & $\mathbf{99}${\tiny $\pm2$} & $70${\tiny $\pm5$} & $90${\tiny $\pm7$} & $\mathbf{96}${\tiny $\pm3$} & $\mathbf{100}${\tiny $\pm0$} \\
			\envname{cube-single-play-task3-v0} & \datasetname{clean} & $\mathbf{99}${\tiny $\pm1$} & $\mathbf{100}${\tiny $\pm0$} & $\mathbf{99}${\tiny $\pm1$} & $82${\tiny $\pm7$} & $\mathbf{97}${\tiny $\pm4$} & $\mathbf{98}${\tiny $\pm3$} & $\mathbf{100}${\tiny $\pm0$} \\
			\envname{cube-single-play-task4-v0} & \datasetname{clean} & $90${\tiny $\pm4$} & $91${\tiny $\pm2$} & $85${\tiny $\pm4$} & $82${\tiny $\pm3$} & $\mathbf{98}${\tiny $\pm2$} & $87${\tiny $\pm2$} & $\mathbf{99}${\tiny $\pm1$} \\
			\envname{cube-single-play-task5-v0} & \datasetname{clean} & $81${\tiny $\pm2$} & $\mathbf{89}${\tiny $\pm4$} & $85${\tiny $\pm4$} & $75${\tiny $\pm6$} & $85${\tiny $\pm3$} & $86${\tiny $\pm13$} & $\mathbf{93}${\tiny $\pm5$} \\
			\midrule
			\envname{cube-single-noisy-task1-v0} & \datasetname{noisy} & $\mathbf{100}${\tiny $\pm0$} & $\mathbf{100}${\tiny $\pm0$} & $\mathbf{100}${\tiny $\pm0$} & $1${\tiny $\pm1$} & $93${\tiny $\pm6$} & $\mathbf{100}${\tiny $\pm0$} & $\mathbf{100}${\tiny $\pm0$} \\
			\envname{cube-single-noisy-task2-v0} & \datasetname{noisy} & $\mathbf{100}${\tiny $\pm0$} & $\mathbf{100}${\tiny $\pm0$} & $\mathbf{99}${\tiny $\pm1$} & $7${\tiny $\pm6$} & $92${\tiny $\pm9$} & $\mathbf{100}${\tiny $\pm0$} & $\mathbf{100}${\tiny $\pm0$} \\
			\envname{cube-single-noisy-task3-v0} & \datasetname{noisy} & $\mathbf{100}${\tiny $\pm0$} & $\mathbf{100}${\tiny $\pm0$} & $\mathbf{100}${\tiny $\pm0$} & $85${\tiny $\pm9$} & $53${\tiny $\pm25$} & $\mathbf{100}${\tiny $\pm0$} & $\mathbf{100}${\tiny $\pm0$} \\
			\envname{cube-single-noisy-task4-v0} & \datasetname{noisy} & $\mathbf{100}${\tiny $\pm0$} & $\mathbf{100}${\tiny $\pm0$} & $\mathbf{100}${\tiny $\pm0$} & $0${\tiny $\pm0$} & $91${\tiny $\pm4$} & $\mathbf{100}${\tiny $\pm0$} & $\mathbf{100}${\tiny $\pm0$} \\
			\envname{cube-single-noisy-task5-v0} & \datasetname{noisy} & $\mathbf{100}${\tiny $\pm0$} & $\mathbf{100}${\tiny $\pm0$} & $\mathbf{99}${\tiny $\pm2$} & $0${\tiny $\pm0$} & $91${\tiny $\pm5$} & $\mathbf{100}${\tiny $\pm0$} & $\mathbf{100}${\tiny $\pm0$} \\
			\midrule
			\envname{cube-double-play-task1-v0} & \datasetname{clean} & $78${\tiny $\pm2$} & $\mathbf{96}${\tiny $\pm2$} & $\mathbf{97}${\tiny $\pm1$} & $19${\tiny $\pm2$} & $\mathbf{95}${\tiny $\pm1$} & $91${\tiny $\pm7$} & $\mathbf{97}${\tiny $\pm1$} \\
			\envname{cube-double-play-task2-v0} & \datasetname{clean} & $52${\tiny $\pm3$} & $83${\tiny $\pm11$} & $83${\tiny $\pm1$} & $15${\tiny $\pm4$} & $63${\tiny $\pm4$} & $83${\tiny $\pm7$} & $\mathbf{89}${\tiny $\pm1$} \\
			\envname{cube-double-play-task3-v0} & \datasetname{clean} & $40${\tiny $\pm6$} & $69${\tiny $\pm11$} & $75${\tiny $\pm6$} & $9${\tiny $\pm7$} & $63${\tiny $\pm17$} & $73${\tiny $\pm7$} & $\mathbf{81}${\tiny $\pm7$} \\
			\envname{cube-double-play-task4-v0} & \datasetname{clean} & $7${\tiny $\pm3$} & $21${\tiny $\pm1$} & $20${\tiny $\pm4$} & $1${\tiny $\pm1$} & $17${\tiny $\pm5$} & $22${\tiny $\pm4$} & $\mathbf{49}${\tiny $\pm12$} \\
			\envname{cube-double-play-task5-v0} & \datasetname{clean} & $37${\tiny $\pm10$} & $\mathbf{81}${\tiny $\pm10$} & $73${\tiny $\pm2$} & $8${\tiny $\pm3$} & $73${\tiny $\pm1$} & $54${\tiny $\pm13$} & $70${\tiny $\pm6$} \\
			\midrule
			\envname{cube-double-noisy-task1-v0} & \datasetname{noisy} & $53${\tiny $\pm17$} & $89${\tiny $\pm4$} & $\mathbf{97}${\tiny $\pm2$} & $1${\tiny $\pm1$} & $87${\tiny $\pm5$} & $80${\tiny $\pm3$} & $\mathbf{92}${\tiny $\pm2$} \\
			\envname{cube-double-noisy-task2-v0} & \datasetname{noisy} & $5${\tiny $\pm2$} & $74${\tiny $\pm3$} & $\mathbf{80}${\tiny $\pm4$} & $0${\tiny $\pm0$} & $\mathbf{79}${\tiny $\pm11$} & $29${\tiny $\pm14$} & $\mathbf{81}${\tiny $\pm10$} \\
			\envname{cube-double-noisy-task3-v0} & \datasetname{noisy} & $5${\tiny $\pm2$} & $55${\tiny $\pm5$} & $53${\tiny $\pm6$} & $0${\tiny $\pm0$} & $\mathbf{60}${\tiny $\pm5$} & $29${\tiny $\pm3$} & $\mathbf{61}${\tiny $\pm12$} \\
			\envname{cube-double-noisy-task4-v0} & \datasetname{noisy} & $1${\tiny $\pm1$} & $43${\tiny $\pm13$} & $51${\tiny $\pm2$} & $1${\tiny $\pm1$} & $58${\tiny $\pm18$} & $9${\tiny $\pm2$} & $\mathbf{65}${\tiny $\pm5$} \\
			\envname{cube-double-noisy-task5-v0} & \datasetname{noisy} & $2${\tiny $\pm3$} & $35${\tiny $\pm4$} & $30${\tiny $\pm4$} & $0${\tiny $\pm0$} & $\mathbf{40}${\tiny $\pm11$} & $11${\tiny $\pm6$} & $\mathbf{38}${\tiny $\pm13$} \\
			\midrule
			\envname{scene-play-task1-v0} & \datasetname{clean} & $\mathbf{100}${\tiny $\pm0$} & $\mathbf{100}${\tiny $\pm0$} & $\mathbf{100}${\tiny $\pm0$} & $\mathbf{99}${\tiny $\pm1$} & $\mathbf{100}${\tiny $\pm0$} & $\mathbf{100}${\tiny $\pm0$} & $\mathbf{100}${\tiny $\pm0$} \\
			\envname{scene-play-task2-v0} & \datasetname{clean} & $\mathbf{100}${\tiny $\pm0$} & $\mathbf{100}${\tiny $\pm0$} & $\mathbf{99}${\tiny $\pm1$} & $3${\tiny $\pm2$} & $\mathbf{100}${\tiny $\pm0$} & $\mathbf{100}${\tiny $\pm0$} & $\mathbf{100}${\tiny $\pm0$} \\
			\envname{scene-play-task3-v0} & \datasetname{clean} & $93${\tiny $\pm5$} & $93${\tiny $\pm5$} & $94${\tiny $\pm3$} & $80${\tiny $\pm4$} & $\mathbf{99}${\tiny $\pm1$} & $93${\tiny $\pm8$} & $\mathbf{100}${\tiny $\pm0$} \\
			\envname{scene-play-task4-v0} & \datasetname{clean} & $66${\tiny $\pm13$} & $16${\tiny $\pm19$} & $7${\tiny $\pm4$} & $3${\tiny $\pm1$} & $\mathbf{100}${\tiny $\pm0$} & $60${\tiny $\pm42$} & $58${\tiny $\pm29$} \\
			\envname{scene-play-task5-v0} & \datasetname{clean} & $\mathbf{13}${\tiny $\pm9$} & $0${\tiny $\pm0$} & $0${\tiny $\pm0$} & $0${\tiny $\pm0$} & $0${\tiny $\pm0$} & $0${\tiny $\pm0$} & $0${\tiny $\pm0$} \\
			\midrule
			\envname{scene-noisy-task1-v0} & \datasetname{noisy} & $86${\tiny $\pm12$} & $\mathbf{100}${\tiny $\pm0$} & $\mathbf{100}${\tiny $\pm0$} & $\mathbf{99}${\tiny $\pm2$} & $\mathbf{100}${\tiny $\pm0$} & $71${\tiny $\pm17$} & $\mathbf{100}${\tiny $\pm0$} \\
			\envname{scene-noisy-task2-v0} & \datasetname{noisy} & $70${\tiny $\pm8$} & $\mathbf{98}${\tiny $\pm2$} & $\mathbf{98}${\tiny $\pm2$} & $0${\tiny $\pm0$} & $\mathbf{96}${\tiny $\pm3$} & $\mathbf{97}${\tiny $\pm4$} & $\mathbf{100}${\tiny $\pm0$} \\
			\envname{scene-noisy-task3-v0} & \datasetname{noisy} & $0${\tiny $\pm0$} & $25${\tiny $\pm6$} & $39${\tiny $\pm16$} & $44${\tiny $\pm4$} & $\mathbf{100}${\tiny $\pm0$} & $45${\tiny $\pm27$} & $80${\tiny $\pm27$} \\
			\envname{scene-noisy-task4-v0} & \datasetname{noisy} & $0${\tiny $\pm0$} & $8${\tiny $\pm11$} & $1${\tiny $\pm1$} & $1${\tiny $\pm1$} & $37${\tiny $\pm39$} & $0${\tiny $\pm0$} & $\mathbf{67}${\tiny $\pm31$} \\
			\envname{scene-noisy-task5-v0} & \datasetname{noisy} & $\mathbf{0}${\tiny $\pm0$} & $\mathbf{0}${\tiny $\pm0$} & $\mathbf{0}${\tiny $\pm0$} & $\mathbf{0}${\tiny $\pm0$} & $\mathbf{0}${\tiny $\pm0$} & $\mathbf{0}${\tiny $\pm0$} & $\mathbf{0}${\tiny $\pm0$} \\
        \bottomrule
    \end{tabular}
	\end{center}
\end{table}

\begin{figure}[!htbp]
    \centering
    \includegraphics[width=\linewidth]{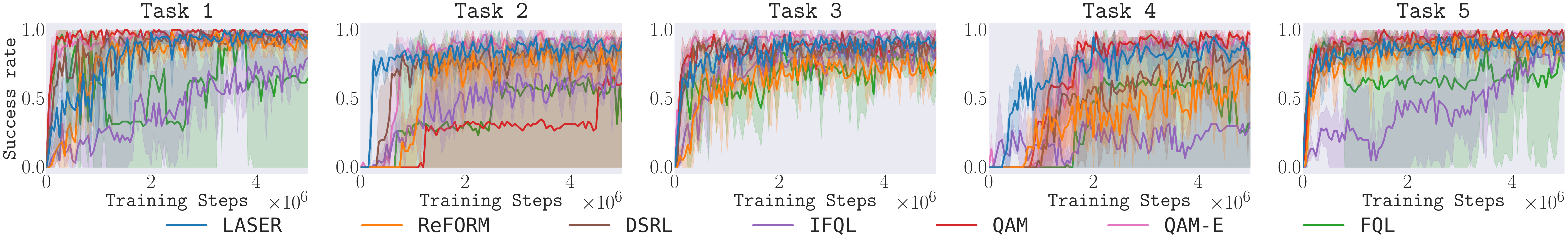}
    \caption{Training curves for the \envname{antmaze-large} environment with the \datasetname{clean} dataset.}
    \label{fig: antmaze-large-navigate-training}
\end{figure}
\begin{figure}[!htbp]
    \centering
    \includegraphics[width=\linewidth]{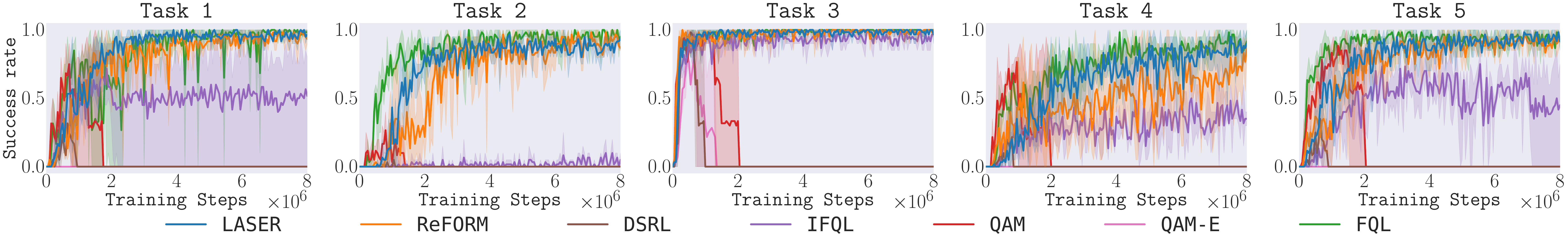}
    \caption{Training curves for the \envname{antmaze-large} environment with the \datasetname{noisy} dataset.}
    \label{fig: antmaze-large-explore-training}
\end{figure}
\begin{figure}[!htbp]
    \centering
    \includegraphics[width=\linewidth]{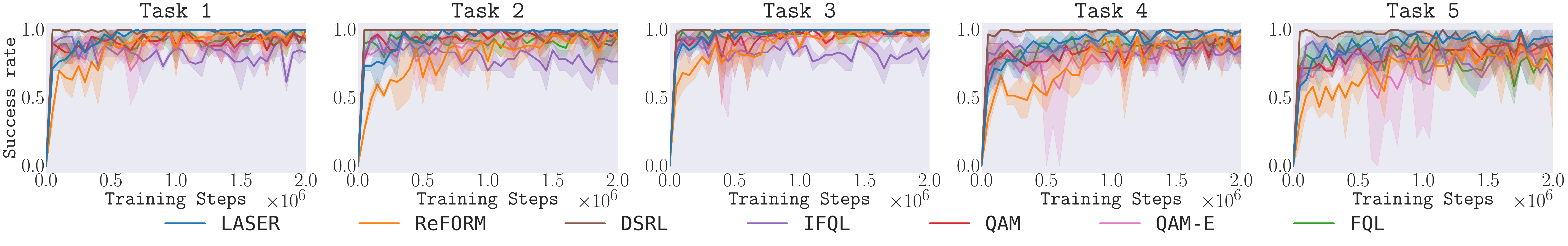}
    \caption{Training curves for the \envname{cube-single} environment with the \datasetname{clean} dataset.}
    \label{fig: cube-single-play-training}
\end{figure}
\begin{figure}[!htbp]
    \centering
    \includegraphics[width=\linewidth]{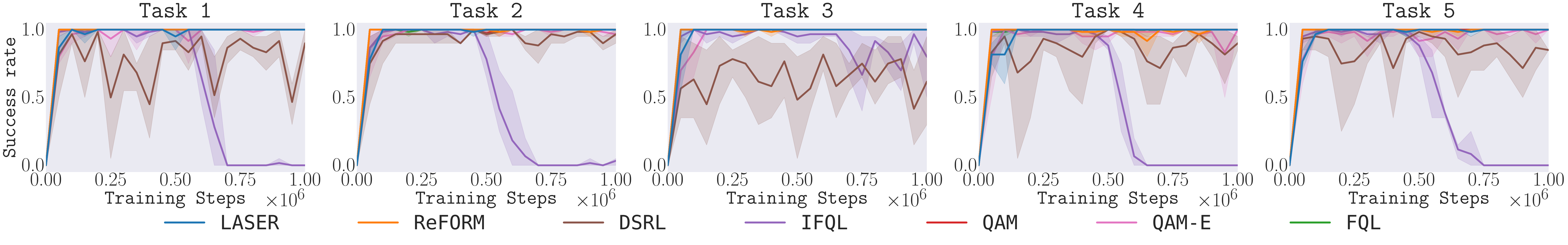}
    \caption{Training curves for the \envname{cube-single} environment with the \datasetname{noisy} dataset.}
    \label{fig: cube-single-noisy-training}
\end{figure}
\begin{figure}[!htbp]
    \centering
    \includegraphics[width=\linewidth]{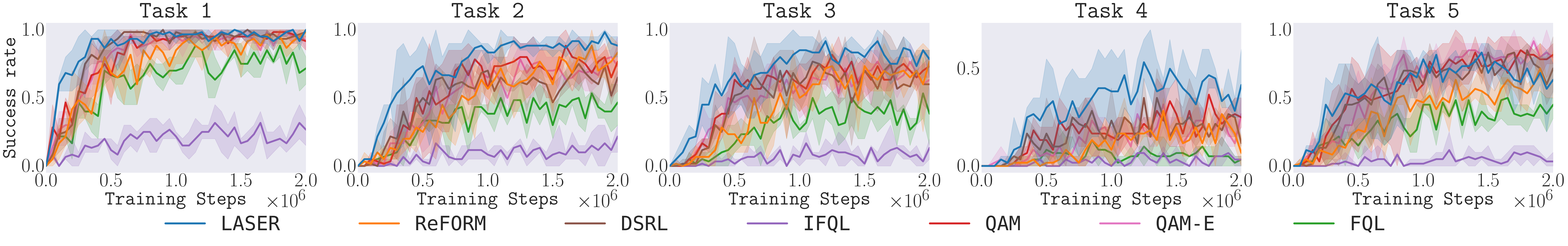}
    \caption{Training curves for the \envname{cube-double} environment with the \datasetname{clean} dataset.}
    \label{fig: cube-double-play-training}
\end{figure}
\begin{figure}[!htbp]
    \centering
    \includegraphics[width=\linewidth]{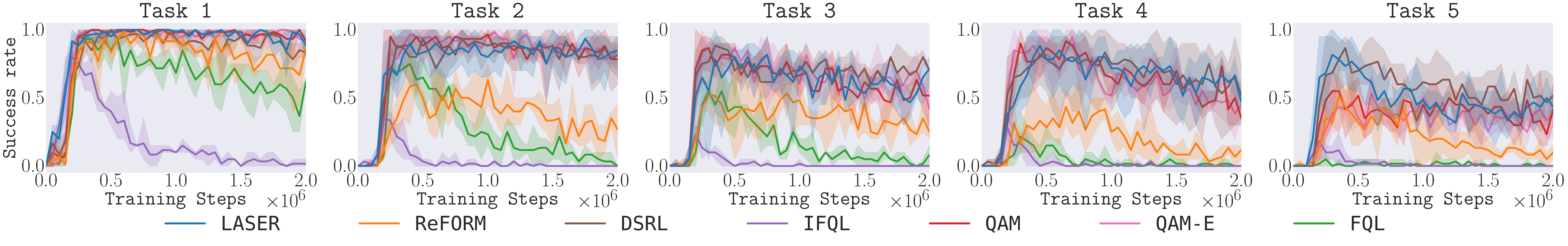}
    \caption{Training curves for the \envname{cube-double} environment with the \datasetname{noisy} dataset.}
    \label{fig: cube-double-noisy-training}
\end{figure}
\begin{figure}[!htbp]
    \centering
    \includegraphics[width=\linewidth]{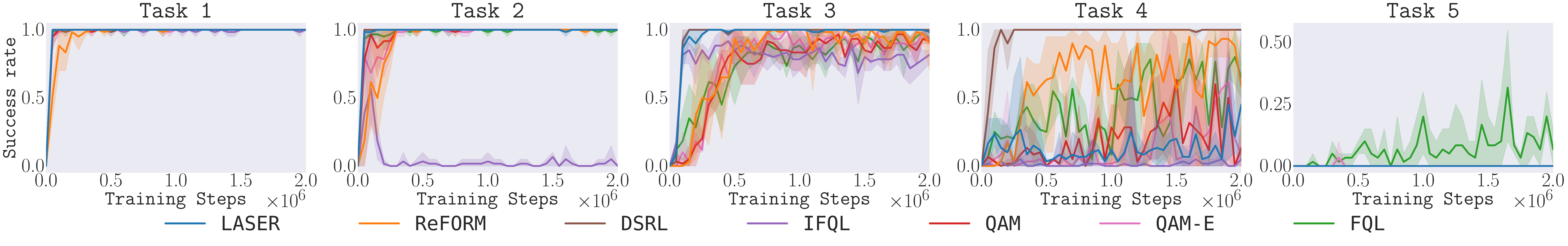}
    \caption{Training curves for the \envname{scene} environment with the \datasetname{clean} dataset.}
    \label{fig: scene-play-training}
\end{figure}
\begin{figure}[!htbp]
    \centering
    \includegraphics[width=\linewidth]{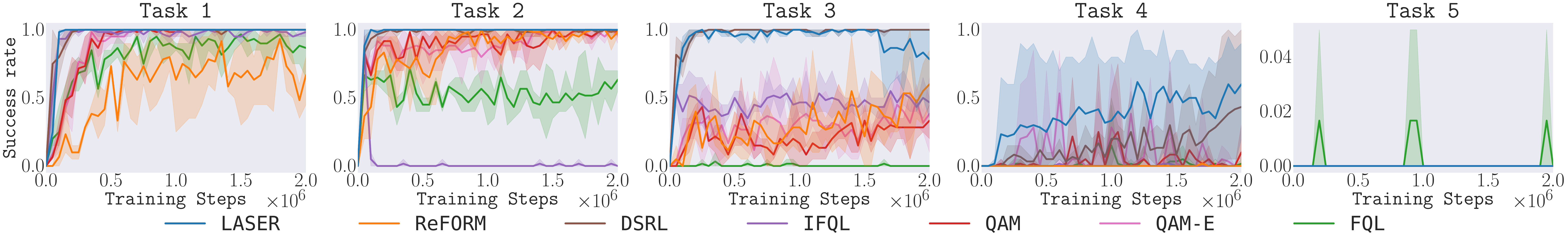}
    \caption{Training curves for the \envname{scene} environment with the \datasetname{noisy} dataset.}
    \label{fig: scene-noisy-training}
\end{figure}

\end{document}